\documentclass{article}
\PassOptionsToPackage{numbers, compress}{natbib}
\usepackage[preprint]{neurips_2026}

\usepackage[utf8]{inputenc} 
\usepackage[T1]{fontenc}    
\usepackage{hyperref}       
\usepackage{url}            
\usepackage{booktabs}       
\usepackage{amsfonts}       
\usepackage{nicefrac}       
\usepackage{microtype}      
\usepackage{xcolor}         

\usepackage{amsmath}
\usepackage{amsthm}
\usepackage{amssymb}
\usepackage{mathtools}
\usepackage{algorithmic}
\usepackage{pifont}
\usepackage[linesnumbered, ruled]{algorithm2e}
\usepackage{booktabs}
\usepackage{diagbox} 
\usepackage{multirow} 
\usepackage{makecell} 
\usepackage{longtable} 
\usepackage{threeparttable}
\usepackage{graphicx}
\usepackage{wrapfig}
\usepackage{subfig}
\usepackage[table]{xcolor}
\newcommand{\G}{\cellcolor{gray!15}}

\theoremstyle{plain}
\newtheorem{theorem}{Theorem}[section]

\newtheorem{lemma}[theorem]{Lemma}

\theoremstyle{definition}
\newtheorem{definition}[theorem]{Definition}
\newtheorem{assumption}[theorem]{Assumption}
\theoremstyle{remark}
\newtheorem{remark}[theorem]{\bf Remark}

\title{Dimensionality Reduction for Robust Federated Learning: A Theoretical Analysis and Convergence Guarantee}

\author{%
  Shiyuan Zuo$^1$ \\
  \texttt{zuoshiyuan@bit.edu.cn} 
  \And 
  Jiashuo Li$^2$ \\
  \texttt{xjtuljs@stu.xjtu.edu.cn}
  \And
  Rongfei Fan$^{1,}$\thanks{Corresponding author.} \\
  \texttt{fanrongfei@bit.edu.cn} \\
    \And
  Han Hu$^1$ \\
  \texttt{hhu@bit.edu.cn}
  \And
  Jie Xu$^3$ \\
  \texttt{xujie@cuhk.edu.cn}
}

\begin{document}

\maketitle

\begin{center}
  \vspace{-30pt}  
  $^1$Beijing Institute of Technology, Beijing, China \quad  
  $^2$Xi'an Jiaotong University, Xi'an, China \quad
  $^3$The Chinese University of Hong Kong (Shenzhen), Shenzhen, China
\end{center}

\vspace{-5pt}

\begin{abstract}

Federated Learning (FL) enables multiple clients to collaboratively train models without sharing raw data, but it is highly vulnerable to Byzantine attacks. Existing robust approaches can neutralize these threats but incur substantial computational overhead during high-dimensional gradient aggregation, an overhead that scales poorly with model size and increasingly dominates the training cost as modern models grow larger. To address this computational bottleneck, we propose Projected Dimensionality Reduction (PDR), a universal acceleration framework for vector-level distance-based robust aggregators, which performs robust aggregation by compressing gradients into a drastically smaller subspace via sparse random projection to efficiently compute reliability weights. This approach reduces the server computational complexity to an optimal $ \mathcal{O}(Mp) $, where $ M $ is the number of clients and $ p $ is the model dimension, matching the theoretical lower bound required merely to read the gradients. We establish convergence guarantees under standard FL assumptions in prior Byzantine-robust FL analyses. By leveraging the Subspace Embedding Theorem, we show that PDR achieves optimal convergence rates of $ \mathcal{O}(1/\sqrt{T}) $ for non-convex functions and $ \mathcal{O}(1/T) $ for strongly convex functions, where $ T $ denotes the number of iterations. Crucially, we mathematically demonstrate that this massive acceleration comes almost for free, merely inflating the inherent Byzantine error floor by a bounded, tunable factor of $ \frac{1+\epsilon}{1-\epsilon} $. Experimental results on benchmark datasets confirm that integrating PDR with existing aggregators yields orders of magnitude speedups in time efficiency while maintaining highly competitive convergence performance.

\end{abstract}
\section{Introduction} \label{sec:intro}

Federated Learning (FL) has emerged as a transformative distributed machine learning paradigm, enabling multiple clients to collaboratively train a global model without sharing their raw, privacy-sensitive data \citep{zuo2024byzantineresilient, konevcny2016federated, li2020federated, dorfman2023docofl, zhao2024huber}. A standard FL protocol operates in iterative rounds: a central server broadcasts the current global model to selected clients, the clients compute local gradient updates using their private data, and the server aggregates these updates to refine the global model \citep{wang2019adaptive, guo2023fedbr}. However, the opaque nature of this process introduces a critical vulnerability: \textit{Byzantine attacks} \citep{zuo2025efficient, chen2017distributed}. In practical FL deployments, malicious or compromised clients can send arbitrary, arbitrarily large, or carefully crafted malicious gradients to the server \citep{yang2020adversary, cao2019distributed}. These Byzantine attacks aim to hijack the global model, degrade its convergence, or implant hidden backdoors, posing a severe threat to the security of distributed FL systems \citep{chen2017distributed}.

To defend against such threats, a plethora of Byzantine-robust aggregation rules (e.g., Krum \citep{blanchard2017machine}, Bulyan \citep{mhamdi2018hidden}, Geometric Median \citep{chen2017distributed}, and MCA \citep{luan2024robust}) have been proposed. These methods typically rely on computing pairwise Euclidean distances \citep{blanchard2017machine}, sorting \citep{yin2018byzantine}, or estimating the variance among the high-dimensional gradient vectors to identify and filter out malicious outliers. While theoretically sound and empirically effective for moderately sized models, these defenses face a growing computational challenge as model scales continue to expand \citep{cao2021fltrust, shejwalkar2022back}. With modern deep networks pushing $p$ from millions toward billions of parameters, the time complexity of distance-based robust aggregators grows super-linearly in the model dimension, causing the server-side aggregation overhead to account for an increasingly large fraction of each communication round. This trend makes the server-side computational efficiency of robust aggregation a pressing concern for the next generation of FL systems, independent of whether any single FL deployment reaches LLM scale. Addressing this computational bottleneck is the focus of this paper.

Recent efforts have attempted to mitigate this curse of dimensionality, but with limited success. For instance, DnC \citep{shejwalkar2021manipulating} or random sampling-based methods attempt to reduce the computational burden by evaluating robustness on random subsets of dimensions or clients. However, these heuristic approximations often sacrifice strict theoretical guarantees, suffer from information loss, and can be easily bypassed by sophisticated attacks that hide malicious perturbations in the unsampled dimensions \citep{baruch2019little, shejwalkar2022back}. On another front, dimensionality reduction techniques, particularly those based on the Johnson-Lindenstrauss (JL) lemma \citep{achlioptas2003database}, have been introduced to FL. Yet, existing literature predominantly leverages the JL lemma for \textit{privacy protection} \citep{blocki2012johnson, li2020privacy}, rather than addressing the computational bottleneck of Byzantine robustness. The potential of random projection to fundamentally accelerate robust aggregation remains largely untapped.

In this paper, we bridge this critical gap. Inspired by the JL lemma and the Subspace Embedding theorem \citep{woodruff2014sketching}, we observe a fundamental property of distance-based robust aggregators: their core outlier-detection mechanism relies solely on the \textit{relative distances} among client updates, rather than the exact high-dimensional coordinates. Motivated by this mathematical insight, we propose a novel and highly efficient framework, Projected Dimensionality Reduction (PDR), universal across all vector-level distance-based robust aggregators. Instead of executing the heavy robust aggregation in the original massive $\mathbb{R}^p$ space, our server first compresses the gradients into a drastically lower-dimensional subspace $\mathbb{R}^k$ ($k \ll p$) via sparse random projection. The robust aggregator then computes the reliability weights entirely in this low-dimensional space. Finally, the global update is reconstructed in the high-dimensional space using these reliability weights. 

Our main contributions are summarized as follows:
\begin{itemize}
    \item \textbf{Universal Acceleration Framework:} We propose a plug-and-play dimensionality reduction framework that shifts the computational bottleneck of vector-level distance-based aggregators from the massive model dimension $p$ to a small projected dimension $k$. We formally characterize the applicable aggregator family in Remark \ref{rm:ragg}, ensuring that our universality claim is precise rather than vacuous.
    \item \textbf{Rigorous Theoretical Guarantees:} Leveraging the Subspace Embedding Theorem, we prove our framework achieves optimal convergence rates of $\mathcal{O}(1/\sqrt{T})$ for non-convex and $\mathcal{O}(1/T)$ for strongly convex functions. Remarkably, this massive acceleration comes at a controlled cost: inflating the inherent Byzantine error floor by a tunable factor of $\frac{1+\epsilon}{1-\epsilon}, \epsilon \in (0, 1)$.
    \item \textbf{Empirical Superiority:} Extensive evaluations on datasets like TinyImageNet, CIFAR100, and CIFAR10 demonstrate that our method yields orders of magnitude speedups in server execution time. Simultaneously, it maintains highly competitive accuracy and neutralizes severe Byzantine threats under non-IID settings.
\end{itemize}

\textbf{Notation:} Everywhere in the text $\left\lVert x \right\rVert$ denotes a standard $\ell_2$-norm of vector $x$, $\langle x, y \rangle$ refers to the standard inner product of vectors $x, y$. 

\section{Methodology} \label{sec:meth}

In this section, we present PDR algorithm. We first introduce the problem setup under Byzantine attacks. 

\subsection{Problem Setup}

\textbf{FL optimization problem:} Consider an FL system consisting of one central server and $M$ clients, indexed by the set $\mathcal{M} \triangleq \{ 1, 2, \dots, M \}$. Each client $m \in \mathcal{M}$ possesses a local dataset $\mathcal{S}_m$ of size $S_m$. The $i$-th data sample in $\mathcal{S}_m$ is denoted as $s_{m,i}$. By leveraging the distributed datasets $\{ \mathcal{S}_m \}_{m \in \mathcal{M}}$, the learning objective is to collaboratively train a $p$-dimensional model parameter vector $w \in \mathbb{R}^p$ that minimizes a global loss function $F(w)$. Specifically, we aim to solve the following optimization problem:
\begin{equation} \label{pro:global}
    \min_{{w} \in \mathbb{R}^p} F({w}) = \frac{1}{\sum_{i \in \mathcal{M}} S_i} \sum_{m
    \in \mathcal{M}} S_m F_m(w), \quad F_m(w) = \frac{1}{S_m} \sum_{i \in \mathcal{S}_m} F_m(w, s_{m ,i}). 
\end{equation}

Accordingly, the gradients of the global and local loss functions with respect to the model parameter $w$ are denoted as $\nabla F(w)$ and $\nabla F_m(w)$, respectively. 

To minimize the global objective, standard FL paradigms iteratively update $w$ by aggregating these local gradients through continuous communication rounds between the central server and the $M$ clients. A fundamental premise of this gradient-based convergence process is that every client acts honestly and provides accurate local computations. However, this idealized assumption completely fails under a Byzantine threat system. The core difficulty in tackling the objective defined in (\ref{pro:global}) stems from the fact that compromised nodes might work collaboratively to upload arbitrary poisoned messages instead of their true gradients, ultimately destroying the integrity of the training procedure.

\textbf{Byzantine attacks:} Within this FL environment, we assume the existence of $B$ malicious clients among the total $M$ clients, which collectively constitute the Byzantine set $\mathcal{B}$. To disrupt the global model optimization, any attacker within $\mathcal{B}$ can inject an arbitrary and potentially destructive vector $\star \in \mathbb{R}^p$ into the server's aggregation process. Let $g_m^t$ denote the exact message transmitted by the $m$-th client at round $t$. The adversarial updates can then be mathematically expressed as:
\begin{equation}
    g_m^t = \star, \quad \text{for } m \in \mathcal{B}. 
\end{equation}

For ease of representation and to evaluate the system's vulnerability, we introduce the metric $b$ to denote the Byzantine attack ratio level. This metric is defined as the data-weighted ratio of the malicious clients, calculated as follows:
\begin{equation}
    b \triangleq \frac{1}{\sum_{i \in \mathcal{M}} S_i} \sum_{m \in \mathcal{B}} S_m. 
\end{equation}

It is worth noting that different robust FL strategies exhibit varying degrees of resilience against the attack level $b$. Nevertheless, a universal prerequisite applies: without the assistance of an auxiliary clean dataset or prior statistical knowledge at the server, successful robust aggregation fundamentally relies on the assumption that benign data dominates the system, inherently requiring $b < 0.5$.

\textbf{Robust aggregator:} Here we follow the $(b,c)$-robustness aggregation framework of \citet{gorbunov2021variance}, which generalizes the definitions proposed by \citet{karimireddy2022byzantine}. 
\begin{definition}[$(b, c)$-Robustness Aggregator] \label{def:agg}
    Consider a set of $M$ vectors $\mathcal{G} = \{ g_1, g_2, \dots, g_M \}$ in $\mathbb{R}^p$. We postulate the existence of a truthful subset $\mathcal{G}_{\text{true}} \subseteq \mathcal{G}$ comprising $G_{\text{true}}$ elements, such that $G_{\text{true}} \geq (1- b)M$ for a given threshold $b \leq b_{\max} < 0.5$. Additionally, let $\nu \geq 0$ be a bounding parameter such that the condition $\max_{i, j \in \mathcal{G}_{\text{true}}} \mathbb{E} \left\lVert g_i - g_j \right\rVert^2 = \nu^2$ holds. Under these premises, a $(b, c)$-Robust Aggregator (abbreviated as RAgg) is characterized by its ability to produce an aggregated vector $\hat{g} = \text{RAgg}(g_1, g_2, \dots, g_M)$ that achieves the following performance guarantee for a specific constant $c > 0$:
    \begin{equation}
        \mathbb{E} \left\lVert \hat{g} - \bar{g} \right\rVert^2 \leq c b \nu^2.
    \end{equation}
    
    Here, $\bar{g} = \frac{1}{G_{\text{true}}} \sum_{i \in \mathcal{G}_{\text{true}}} g_i$ denotes the exact mean of the truthful subset. It is crucial to emphasize that the computation of $\hat{g}$ is entirely independent of the knowledge of $\nu^2$. Furthermore, this $(b, c)$-RAgg formulation can be readily extended to accommodate non-uniform aggregation weights for the vectors in $\mathcal{G}$; the proof for this extension is trivial and is therefore omitted.
\end{definition}

\subsection{Algorithm}

To achieve our design goal of efficiently defending against Byzantine attacks, we first introduce the Subspace Embedding Theorem \citep{woodruff2014sketching}, which guarantees that the distance relationships essential for detecting malicious updates are strictly preserved even after significant dimensionality reduction.

\begin{lemma}[Subspace Embedding Theorem] \label{lm:subemb}
    Let $\mathcal{V} = \text{span}\{g_1, g_2, \dots, g_M\}$ be a linear subspace of dimension $M$ in $\mathbb{R}^p$ spanned by the client gradients. Let $P \in \mathbb{R}^{k \times p}$ be a random projection matrix whose entries are independent, zero mean sub-Gaussian random variables with variance $1/k$ (for example, scaled Gaussian or Rademacher distributions). Given an error tolerance $\epsilon \in (0, 1)$ and a failure probability $\delta \in (0, 1)$, if the target dimension $k$ satisfies
    \begin{equation}
        k \geq \frac{18}{\epsilon^2} \left( M + 2 \ln\left(\frac{2}{\delta}\right) \right),
    \end{equation}
    then with probability at least $1 - \delta$, the matrix $P$ preserves the Euclidean norm of \textbf{all} vectors in the entire subspace $\mathcal{V}$ simultaneously. That is, for any vector $v \in \mathcal{V}$:
    \begin{equation}
        (1 - \epsilon) \left\lVert v \right\rVert_2^2 \leq \left\lVert Pv \right\rVert_2^2 \leq (1 + \epsilon) \left\lVert v \right\rVert_2^2.
    \end{equation}
\end{lemma}


\begin{remark}[Construction of the Random Projection Matrix $P$]
    In practice, there are two standard approaches to construct a random projection matrix $P \in \mathbb{R}^{k \times p}$ that perfectly satisfies Lemma \ref{lm:subemb}:
    \begin{itemize}
        \item \textbf{Gaussian Random Projection:} Each entry $P_{i,j}$ is independently sampled from a zero-mean normal distribution with variance $1/k$, i.e., $P_{i,j} \sim \mathcal{N}(0, 1/k)$.
        \item \textbf{Sparse Random Projection (Achlioptas Matrix) \citep{woodruff2014sketching, achlioptas2003database}:} To significantly accelerate computations, $P_{i,j}$ can be generated from a sparse distribution:
        \begin{equation}
            P_{i,j} = \sqrt{\frac{s}{k}} 
            \begin{cases} 
                +1, & \text{with probability } \frac{1}{2s}, \\ 
                0, & \text{with probability } 1 - \frac{1}{s}, \\ 
                -1, & \text{with probability } \frac{1}{2s}.
            \end{cases}
        \end{equation}
    \end{itemize}
\end{remark}

During the $t$-th communication round, an honest client $m \in \mathcal{M} \setminus \mathcal{B}$ computes and uploads its true local stochastic gradient $g_m^t = \nabla F_m(w^t, \xi_m^t)$ to the central server, where $\xi_m^t \subseteq \mathcal{S}_m$ denotes the sampled mini-batch. Conversely, a Byzantine client $m \in \mathcal{B}$ may transmit an arbitrary poisoned vector $g_m^t = \star$ to derail the learning process. 

Upon receiving the full set of high-dimensional updates $\{g_m^t\}_{m \in \mathcal{M}}$, the central server does not execute the robust aggregation algorithm directly in $\mathbb{R}^p$, as doing so is computationally prohibitive. Instead, motivated by Lemma \ref{lm:subemb}, the server generates a random projection matrix $P^t \in \mathbb{R}^{k \times p}$ (which can be regenerated at each round to resist adaptive attacks) to compress these vectors, yielding their low-dimensional representations $\tilde{g}_m^t = P^t g_m^t$. 

Subsequently, the server feeds these compressed vectors $\{\tilde{g}_m^t\}_{m \in \mathcal{M}}$ into a robust aggregator (RAgg). By evaluating the distance-based reliability of these low-dimensional vectors, the RAgg assigns a reliability weight $\tilde{\alpha}_m^t$ to each client $m$. Once the reliability weights $\{\tilde{\alpha}_1^t, \tilde{\alpha}_2^t, \dots, \tilde{\alpha}_M^t\}$ are obtained, the server applies them to the \textit{original high-dimensional} vectors to compute the final aggregated gradient $g^t$:
\begin{equation} \label{equ:aggg}
    g^t = \sum_{m \in \mathcal{M}} \tilde{\alpha}_m^t g_m^t. 
\end{equation}
Note that the reliability weights generated by the robust aggregator strictly satisfy $\sum_{m \in \mathcal{M}} \tilde{\alpha}_m^t = 1$ and $\tilde{\alpha}_m^t \geq 0$. This convex combination property is crucial to ensure that the aggregated gradient $g^t$ preserves the original scale of the local updates, thereby preventing unintended scaling of the learning rate.

Finally, the central server updates the global model parameter $w^{t+1}$ by
\begin{equation}
    w^{t+1} = w^t - \eta^t g^t, 
\end{equation}
where $\eta^t$ denotes the learning rate at round $t$. Then the central server broadcasts the global model parameter $w^{t+1}$ to all clients. And we provide a full description of PDR algorithm in Algorithm \ref{alg:trian process}. 

\begin{remark}[Scope of PDR algorithm] \label{rm:ragg}
    It is crucial to clarify that our PDR algorithm is specifically designed to accelerate \textbf{vector-level distance-based robust aggregators}, covering Krum, Geometric Median (GM), Bulyan, and other distance-reliant vector-wise aggregation methods. These algorithms treat each client's update as an indivisible entity and assign a global scalar weight $\tilde{\alpha}_m^t$ based solely on pairwise Euclidean distances. Notably, for GM, the aggregation is practically solved via the Weiszfeld algorithm, where the weights are inversely proportional to pairwise distances. This distance-based property is perfectly preserved by the Subspace Embedding Theorem in the compressed space. In contrast, our framework does not directly apply to \textbf{coordinate-level} aggregators (e.g., Coordinate-wise Median or Trimmed Mean). Because random projection densely mixes all dimensions, it inevitably destroys the coordinate-independent anomaly information upon which these methods rely.
\end{remark}
\section{Theoretical Results} \label{sec:theo}

In this section, we theoretically analyze the robustness and convergence performance of PDR algorithm on non-IID settings. Below, we firstly present the necessary assumptions. 

\subsection{Assumptions} \label{sec:ass}

First, we state some general assumptions which are very common in the convergence proofs. 

\begin{assumption}[Lipschitz Continuity] \label{ass:smooth}
  The loss function $f({w},s_{m,i})$ has $L$-Lipschitz continuity, i.e., for $\forall {w}_1, {w}_2 \in \mathbb{R}^p$, it follows that
  \begin{align}
    f({w}_1,s_{m,i}) - f({w}_2,s_{m,i}) 
    \leqslant \left\langle \nabla f({w}_2,s_{m,i}), {w}_1 - {w}_2 \right\rangle + \frac{L}{2} \left\lVert {w}_1 - {w}_2 \right\rVert^2. 
  \end{align}

  Under Assumption \ref{ass:smooth}, given that the local loss functions $F_m({w})$ are finite averages of the loss function $f({w},s_{m,i})$ , it can be rigorously deduced that they all exhibit Lipschitz continuity. By an analogous line of reasoning, the global loss function $F(w)$ can also be shown to satisfy Lipschitz continuity.  
\end{assumption}

\begin{assumption}[$\mu$-Strong Convexity] \label{ass:convex}
    The loss function $f({w},s_{m,i})$ has $\mu$-strong convexity, i.e., for $\forall {w}_1, {w}_2 \in \mathbb{R}^p$, it follows that
    \begin{equation}
        f({w}_1,s_{m,i}) - f({w}_2,s_{m,i}) \geq \left\langle \nabla f({w}_2,s_{m,i}), {w}_1 - {w}_2 \right\rangle + \frac{\mu}{2} \left\lVert {w}_1 - {w}_2 \right\rVert^2. 
    \end{equation}
    
    Under Assumption \ref{ass:convex}, given that the local loss functions $F_m({w})$ are finite averages of the loss function $f({w},s_{m,i})$, it can be rigorously deduced that they all exhibit $\mu$-strong convexity. By an analogous line of reasoning, the global loss function $F(w)$ can also be shown to satisfy $\mu$-strong convexity.
\end{assumption}

\begin{assumption}[Local Unbiased Gradient] \label{ass:unbiased}
  For $\xi_m \subseteq \mathcal{S}_m$, the gradient of local training loss function $F_m({w},\xi_m)$ is unbiased, which implies that 
  \begin{equation}
    \mathbb{E} \{ \nabla F_m({w},\xi_m) \} = \nabla F_m({w}), m \in \mathcal{M} \setminus \mathcal{B}.
  \end{equation}

  This assumption is widely adopted in the literature \cite{huang2023achieving, wu2023anchor, xiao2023communication}, and $\nabla F_m({w}, \xi_m)$ reduces to the exact gradient $\nabla F_m({w})$ when $\xi_m = \mathcal{S}_m$.
\end{assumption}

\begin{assumption}[Bounded Inner Error] \label{ass:variance}
  For $\forall {w} \in \mathbb{R}^p$, the inner error of gradients is uniformly bounded, i.e., 
  \begin{equation}
    \mathbb{E} \left\{ \left\lVert \nabla F_m({w}, \xi_m) - \nabla F_m({w}) \right\rVert^2 \right\} \leqslant \sigma^2, m \in \mathcal{M} \setminus \mathcal{B}.
  \end{equation}

  This assumption is also assumed in \citet{huang2023achieving, wu2023anchor, xiao2023communication}. 
\end{assumption}

\begin{assumption}[Bounded Data Heterogeneity] \label{ass:heterogeneity}
  We define a representation of data heterogeneity by inspecting gradient direction, i.e., 
  \begin{equation}
    \kappa_m^2 = \left\lVert \nabla F_m({w}) - \nabla F({w}) \right\rVert^2, m \in \mathcal{M} \setminus \mathcal{B}.
  \end{equation}
  We also assume the data heterogeneity is bounded, which implies
  \begin{equation}
    \kappa_m^2 \leqslant \kappa^2, m \in \mathcal{M} \setminus \mathcal{B},
  \end{equation}
  where $\kappa^2$ is the heterogeneity upper bound.

  This assumption is also assumed in \citet{huang2023achieving, wu2023anchor, xiao2023communication}. 
\end{assumption}



\subsection{Convergence Analysis}

In this subsection, we present the convergence analysis of our proposed PDR algorithm for non-convex loss functions that satisfy Assumption \ref{ass:smooth}. All detailed proofs are deferred to Appendix \ref{app:lmmodg}. 

\begin{lemma} \label{lm:modg}
    Based on Definition \ref{def:agg}, Lemma \ref{lm:subemb}, and Assumptions \ref{ass:unbiased}, \ref{ass:variance}, and \ref{ass:heterogeneity}, if the random projection matrix $P^t \in \mathbb{R}^{k \times p}$ satisfies $k \geq \frac{18}{\epsilon^2} \left( M + 2 \ln\left(\frac{2}{\delta}\right) \right)$, then with probability at least $1- \delta$, the following inequality holds for the aggregated gradient $g^t$:
    \begin{equation}
        \mathbb{E} \left\lVert g^t - \bar{g}^t \right\rVert^2 \leq 2cb \frac{1+\epsilon}{1-\epsilon} (\sigma^2 + 2\kappa^2),
    \end{equation}
    where $c$ is a constant determined by the robust aggregator, and $\bar{g}^t$ represents the ground-truth aggregated gradient computed exclusively from the honest clients, defined by:
    \begin{equation} \label{equ:barg}
        \bar{g}^t = \frac{1}{\sum_{i \in \mathcal{M} \setminus \mathcal{B}} S_i} \sum_{m \in \mathcal{M} \setminus \mathcal{B}} S_m \cdot g_m^t. 
    \end{equation}
\end{lemma}

\begin{proof}
    Please refer to Appendix \ref{app:lmmodg}. 
\end{proof}

\begin{remark}[Theoretical Implications of Lemma \ref{lm:modg}]
    The bound established in Lemma \ref{lm:modg} provides two critical theoretical insights into our method:
    \begin{itemize}
        \item \textbf{Controlled Acceleration Cost:} The dimensionality reduction scales the inherent robustness bound ($cb$) by a factor of $\frac{1+\epsilon}{1-\epsilon}$, , tunable via $k$. By setting a moderate $k$, we achieve $\mathcal{O}(p) \to \mathcal{O}(k)$ acceleration with empirically small impact on statistical fidelity.
        \item \textbf{Universal Applicability:} The abstraction of the aggregator's inherent error into the constant $cb$ makes this lemma a plug-and-play guarantee applicable to any aggregator characterized in Remark \ref{rm:ragg}. 
    \end{itemize}
\end{remark}

\begin{theorem} \label{thm:smooth}
    Under Assumptions \ref{ass:smooth}, \ref{ass:unbiased}, \ref{ass:variance}, and \ref{ass:heterogeneity}, suppose the learning rate satisfies $\eta^t \leq \frac{1}{L}$ for all $t \in \{0, 1, \dots, T-1\}$. For a given error tolerance $\epsilon \in (0, 1)$ and a global failure probability $\delta \in (0, 1)$, if the target dimension of the random projection matrices $P^t$ satisfies:
    \begin{equation}
        k \geq \frac{18}{\epsilon^2} \left( M + 2 \ln\left(\frac{2T}{\delta}\right) \right),
    \end{equation}
    then, with probability at least $1 - \delta$ over the randomness of the projection matrices, the sequence of iterates generated over $T$ communication rounds satisfies:
    \begin{equation}
        \frac{1}{\sum_{t=0}^{T-1} \eta^t} \sum_{t=0}^{T-1} \eta^t \mathbb{E} \left\lVert \nabla F(w^t) \right\rVert^2 \leq \frac{2 (F(w^0) - F(w^T))}{\sum_{t=0}^{T-1} \eta^t} + 4cb \frac{1 + \epsilon}{1 - \epsilon} (\sigma^2 + 2 \kappa^2)  + 2(\sigma^2 + \kappa^2). 
    \end{equation}
\end{theorem}

\begin{proof}
    Please refer to Appendix \ref{app:smooth}. 
\end{proof}

\begin{remark} 
    Theorem \ref{thm:smooth} characterizes the training dynamics of our framework, providing two key theoretical insights:
    \begin{itemize}
        \item \textbf{Convergence Rate:} By setting a constant learning rate $\eta^t = \frac{1}{L\sqrt{T}}$ (which naturally satisfies $\eta^t \leq \frac{1}{L}$), the average squared gradient norm $\frac{1}{T} \sum_{t=0}^{T-1} \mathbb{E} \lVert \nabla F(w^t) \rVert^2$ is bounded by $\mathcal{O}(1/\sqrt{T}) + 4cb \frac{1 + \epsilon}{1 - \epsilon} (\sigma^2 + 2 \kappa^2)  + 2(\sigma^2 + \kappa^2)$. 
        \item \textbf{Optimality Gap:} As $T \to \infty$, the optimization term $\mathcal{O}(1/\sqrt{T})$ vanishes, and the algorithm converges to a stable error neighborhood:
        \begin{equation} \label{equ:opt_gap}
            \underbrace{4cb \frac{1 + \epsilon}{1 - \epsilon} (\sigma^2 + 2 \kappa^2)}_{\text{Projected Byzantine Error}} + \underbrace{2(\sigma^2 + \kappa^2)}_{\text{Statistical Sampling Error}}. \nonumber
        \end{equation}
        Crucially, this gap depends strictly on honest data dispersion ($\sigma^2, \kappa^2$) and the aggregator's robustness ($cb$), proving our framework's strict immunity to arbitrary-magnitude Byzantine attacks.
    \end{itemize}
\end{remark}

\begin{theorem} \label{thm:convex}
    Under Assumptions \ref{ass:smooth}, \ref{ass:convex}, \ref{ass:unbiased}, \ref{ass:variance}, and \ref{ass:heterogeneity}, for a given error tolerance $\epsilon \in (0, 1)$ and a global failure probability $\delta \in (0, 1)$, suppose the target dimension of the random projection matrices $P^t$ satisfies:
    \begin{equation}
        k \geq \frac{18}{\epsilon^2} \left( M + 2 \ln\left(\frac{2T}{\delta}\right) \right).
    \end{equation}
    Then, with probability at least $1 - \delta$ over the randomness of the projection matrices, for any learning rate $\eta^t \leq \frac{1}{2L}$, the expected distance to the unique global minimum $w^*$ at any communication round $t$ satisfies the following one-step recurrence:
    \begin{align}
        &\quad \mathbb{E} \left\lVert w^{t+1} - w^* \right\rVert^2 \nonumber \\ 
        &\leq \left( 1 - \frac{\eta^t\mu}{2} \right) \mathbb{E} \left\lVert w^t - w^* \right\rVert^2 + \left( \frac{2\eta^t}{\mu} + 2(\eta^t)^2 \right) \left( 4cb \frac{1 + \epsilon}{1 - \epsilon} (\sigma^2 + 2\kappa^2) + 2(\sigma^2 + \kappa^2) \right).
    \end{align}
    Furthermore, by employing a decaying learning rate $\eta^t = \frac{2}{\mu(t + \gamma)}$ with $\gamma = \frac{4L}{\mu}$, the sequence of iterates generated over $T$ communication rounds achieves the following final convergence bound:
    \begin{equation}
        \mathbb{E} \left\lVert w^T - w^* \right\rVert^2 \leq \frac{\gamma-1}{T+\gamma-1} \left\lVert w^0 - w^* \right\rVert^2 + \frac{8}{\mu^2} \left( 4cb \frac{1 + \epsilon}{1 - \epsilon} (\sigma^2 + 2\kappa^2) + 2(\sigma^2 + \kappa^2) \right).
    \end{equation}
\end{theorem}

\begin{proof}
    Please refer to Appendix \ref{app:convex}.
\end{proof}

\begin{remark}
    Theorem \ref{thm:convex} provides two critical theoretical insights for our framework under strong convexity:
    \begin{itemize}
        \item \textbf{Convergence Rate:} By employing the decaying learning rate $\eta^t = \mathcal{O}(1/t)$, the distance between the initial model $w^0$ and the global optimum $w^*$ decays at a rate of $\mathcal{O}(1/T)$.
        \item \textbf{Optimality Gap:} As $T \to \infty$, the first term vanishes, and the model converges to a stable error neighborhood bounded by $\frac{8}{\mu^2} \left( 4cb \frac{1 + \epsilon}{1 - \epsilon} (\sigma^2 + 2\kappa^2) + 2(\sigma^2 + \kappa^2) \right)$. This non-vanishing error floor is an inherent property of robust aggregation due to the unavoidable bias introduced by Byzantine clients and non-IID data ($\kappa^2$).
    \end{itemize}
\end{remark}

\subsection{Time Complexity Analysis}

Based on the decoupled design of our framework, the server-side robust aggregation consists of three highly efficient steps: (1) \textbf{Sparse Random Projection}, taking $\mathcal{O}(Mp)$ time using a sparse Achlioptas matrix; (2) \textbf{Low-Dimensional Aggregation}, taking $\mathcal{O}(C_{\text{RAgg}}(M, k))$ where $C_{\text{RAgg}}$ is the inherent complexity of the chosen base aggregator; and (3) \textbf{High-Dimensional Reconstruction}, taking $\mathcal{O}(Mp)$ for the final weighted sum. 

Therefore, the universal time complexity of our framework is strictly bounded by:
\begin{equation}
    \mathcal{O}(\text{Total}) = \mathcal{O}\Big(Mp + C_{\text{RAgg}}(M, k)\Big).
\end{equation}

In modern FL with massive model parameters, traditional distance-based aggregators are computationally prohibitive. Our framework shifts the dominant bottleneck from $p$ to $k$, achieving massive acceleration. For instance:
\begin{itemize}
    \item \textbf{Krum:} The original computational complexity drops from $\mathcal{O}(M^2 p)$ to $\mathcal{O}(Mp + M^2 k)$. 
    \item \textbf{Geometric Median:} The complexity drops from $\mathcal{O}(I \cdot M p)$ to $\mathcal{O}(Mp + I \cdot M k)$, where $I$ is the number of Weiszfeld iterations. Moreover, the projection reduces to a single batched matrix multiplication that is fully parallelizable on modern accelerators, whereas Weiszfeld is inherently sequential, further widening the wall-time gap.
\end{itemize}

Ultimately, our framework collapses the overall computational overhead to $\mathcal{O}(Mp)$. Since both $Mk \ll p$ and $Ik \ll p$ hold under the parameter scale of modern deep learning models, this matches the absolute theoretical lower bound required merely to read the $M$ uploaded high-dimensional gradients into the server's memory, rendering the robust aggregation process asymptotically matches the $\mathcal{O}(Mp)$ read lower bound. 
\section{Experiments} \label{sec:resu}

\subsection{Setups}

\textbf{Experimental Setups:} Our empirical evaluations are conducted on TinyImageNet, CIFAR100, and CIFAR10 datasets, utilizing MobileNetV3 \citep{howard2019searching}, VGG16 \citep{simonyan2014very}, and ResNet18 \citep{he2016deep} models. To emulate Non-IID scenarios, the training data is distributed among $M=50$ clients following a Dirichlet distribution with concentration parameters $\beta \in \{0.2, 0.6\}$. All experiments are executed for 100 communication rounds with a local batch size of 32. For our proposed PDR algorithm, the projection parameters are set to $ k = 4096 $ and $ s = 8 $ by default. 

\textbf{Byzantine Scenarios:} We assess the robustness against four representative threat models: Gaussian, Sign-flip, LIE \citep{baruch2019little}, and FoE \citep{xie2020fall} attacks. The malicious client ratio $b$ is systematically scaled across $\{0.1, 0.3\}$ to probe the defense limits. Further implementation details for both the hyperparameters and the attack formulations are provided in Appendix \ref{app:set}.

\textbf{Baselines:} We benchmark our proposed algorithm against four established robust aggregation rules under Byzantine attacks: Krum \citep{blanchard2017machine}, Bulyan \citep{mhamdi2018hidden}, Geometric Median \citep{chen2017distributed}, and MCA \citep{luan2024robust}. We directly integrate our PDR algorithm with these four algorithms to form combined approaches, denoted as PDR+Krum, PDR+Bulyan, PDR+Geometric Median, and PDR+MCA. 

\subsection{Results}

\begin{table}[tbp]
\caption{The maximum test accuracy (\%) and wall time (s) for our method PDR and baselines with $\beta = 0.6$ on TinyImageNet and CIFAR10 datasets. "-" signifies non-convergence of the result.}
\label{tab:0.6}
\resizebox{\textwidth}{!}{
\renewcommand{\arraystretch}{1.2}
\begin{tabular}{ccc|cccc|cccc}
\Xhline{1.2pt}
\multirow{3}{*}{Methods} & \multirow{3}{*}{Attack Name} & Dataset & \multicolumn{4}{c|}{TinyImageNet} & \multicolumn{4}{c}{CIFAR10} \\ \cline{4-11}
 & & b & \multicolumn{2}{c}{0.1} & \multicolumn{2}{c|}{0.3} & \multicolumn{2}{c}{0.1} & \multicolumn{2}{c}{0.3} \\ \cline{4-11}
 & & & Test Accuracy & Wall Time & Test Accuracy & Wall Time & Test Accuracy & Wall Time & Test Accuracy & Wall Time \\ \hline
\multirow{8}{*}{\begin{tabular}[c]{@{}c@{}}Krum\\ PDR+Krum\end{tabular}}
 & \multicolumn{2}{c|}{\multirow{2}{*}{Gaussian Attack}}  & \G 32.32 & \G 8.7583 & \G 32.15 & \G 8.7092 & \G 67.48 & \G 7.9253 & \G 63.09 & \G 7.8852 \\
 & \multicolumn{2}{c|}{}                                  & 32.71 & \textbf{0.2373} & 31.73 & \textbf{0.1841} & 67.15 & \textbf{0.0678} & 67.23 & \textbf{0.0682} \\
 & \multicolumn{2}{c|}{\multirow{2}{*}{Sign-flip Attack}} & \G 31.92 & \G 8.7113 & \G 32.49 & \G 9.0235 & \G 65.37 & \G 7.9376 & \G 66.98 & \G 7.8883 \\
 & \multicolumn{2}{c|}{}                                  & 32.4 & \textbf{0.1867} & 33.46 & \textbf{0.1862} & 66.35 & \textbf{0.0699} & 68.31 & \textbf{0.0707} \\
 & \multicolumn{2}{c|}{\multirow{2}{*}{LIE Attack}}       & \G 32.28 & \G 8.7 & \G 31.58 & \G 8.6914 & \G 61.68 & \G 8.0373 & \G 62.35 & \G 7.8848 \\
 & \multicolumn{2}{c|}{}                                  & 32.94 & \textbf{0.1685} & 29.64 & \textbf{0.1669} & 68.67 & \textbf{0.0663} & 62.09 & \textbf{0.0661} \\
 & \multicolumn{2}{c|}{\multirow{2}{*}{FoE Attack}}       & \G - & \G 8.8335 & \G - & \G 8.7311 & \G 66.37 & \G 7.8762 & \G 64.35 & \G 8.5986 \\
 & \multicolumn{2}{c|}{}                                  & - & \textbf{0.1734} & - & \textbf{0.1741} & 64.84 & \textbf{0.0704} & 74.58 & \textbf{0.0721} \\ \hline
\multirow{8}{*}{\begin{tabular}[c]{@{}c@{}}Bulyan\\ PDR+Bulyan\end{tabular}}
 & \multicolumn{2}{c|}{\multirow{2}{*}{Gaussian Attack}}  & \G 54.47 & \G 9.3945 & \G 50.76 & \G 9.0512 & \G 68.54 & \G 8.0798 & \G 69.26 & \G 8.0094 \\
 & \multicolumn{2}{c|}{}                                  & 54.2 & \textbf{0.2692} & 46.01 & \textbf{0.217} & 70.5 & \textbf{0.1066} & 67.26 & \textbf{0.0978} \\
 & \multicolumn{2}{c|}{\multirow{2}{*}{Sign-flip Attack}} & \G 54.38 & \G 9.5733 & \G 48.99 & \G 9.0253 & \G 69.53 & \G 8.0554 & \G 70 & \G 8.2054 \\
 & \multicolumn{2}{c|}{}                                  & 54.09 & \textbf{0.2545} & 43.83 & \textbf{0.2399} & 69.72 & \textbf{0.1099} & 70.08 & \textbf{0.0979} \\
 & \multicolumn{2}{c|}{\multirow{2}{*}{LIE Attack}}       & \G 54.57 & \G 9.382 & \G 49.66 & \G 9.1335 & \G 70.68 & \G 8.0874 & \G 70.92 & \G 8.3259 \\
 & \multicolumn{2}{c|}{}                                  & 53.74 & \textbf{0.2317} & 42.74 & \textbf{0.217} & 68.79 & \textbf{0.1054} & 72.46 & \textbf{0.0938} \\
 & \multicolumn{2}{c|}{\multirow{2}{*}{FoE Attack}}       & \G 52.48 & \G 9.4193 & \G - & \G 9.0505 & \G 70.8 & \G 8.1047 & \G 56.65 & \G 7.987 \\
 & \multicolumn{2}{c|}{}                                  & 51.97 & \textbf{0.2362} & - & \textbf{0.227} & 71.83 & \textbf{0.109} & 60.6 & \textbf{0.0986} \\ \hline
\multirow{8}{*}{\begin{tabular}[c]{@{}c@{}}Geometric Median\\ PDR+Geometric Median\end{tabular}}
 & \multicolumn{2}{c|}{\multirow{2}{*}{Gaussian Attack}}  & \G 54.33 & \G 2.8878 & \G 54.32 & \G 2.9474 & \G 69.19 & \G 0.9569 & \G 69.07 & \G 2.2157 \\
 & \multicolumn{2}{c|}{}                                  & 54.45 & \textbf{0.2752} & 54.36 & \textbf{0.2509} & 69.76 & \textbf{0.1082} & 69.25 & \textbf{0.1395} \\
 & \multicolumn{2}{c|}{\multirow{2}{*}{Sign-flip Attack}} & \G 51.25 & \G 2.856 & \G - & \G 3.8209 & \G 72.25 & \G 1.6069 & \G 67.02 & \G 1.4173 \\
 & \multicolumn{2}{c|}{}                                  & 50.44 & \textbf{0.2642} & - & \textbf{0.3264} & 70.49 & \textbf{0.1570} & 68.77 & \textbf{0.2260} \\
 & \multicolumn{2}{c|}{\multirow{2}{*}{LIE Attack}}       & \G 54.83 & \G 2.8669 & \G 54.16 & \G 2.8961 & \G 69.34 & \G 1.5592 & \G 68.78 & \G 1.3945 \\
 & \multicolumn{2}{c|}{}                                  & 54.54 & \textbf{0.2392} & 54.12 & \textbf{0.2447} & 72.93 & \textbf{0.1716} & 72.84 & \textbf{0.1390} \\
 & \multicolumn{2}{c|}{\multirow{2}{*}{FoE Attack}}       & \G 50.35 & \G 5.0288 & \G - & \G 21.4953 & \G 70.55 & \G 0.9417 & \G 67.52 & \G 2.4252 \\
 & \multicolumn{2}{c|}{}                                  & 50.96 & \textbf{0.3643} & - & \textbf{1.3499} & 72.06 & \textbf{0.1408} & 63.46 & \textbf{0.3559} \\ \hline
\multirow{8}{*}{\begin{tabular}[c]{@{}c@{}}MCA\\ PDR+MCA\end{tabular}}
 & \multicolumn{2}{c|}{\multirow{2}{*}{Gaussian Attack}}  & \G 54.49 & \G 1.1212 & \G 54.26 & \G 1.1583 & \G 72.34 & \G 1.0576 & \G 69.62 & \G 1.3444 \\
 & \multicolumn{2}{c|}{}                                  & 54.43 & \textbf{0.2613} & 54.3 & \textbf{0.2249} & 71.2 & \textbf{0.3356} & 69.31 & \textbf{0.1448} \\
 & \multicolumn{2}{c|}{\multirow{2}{*}{Sign-flip Attack}} & \G 54.82 & \G 1.3847 & \G - & \G 219.5042 & \G 70.23 & \G 0.1985 & \G - & \G 0.6459 \\
 & \multicolumn{2}{c|}{}                                  & 54.72 & \textbf{0.2529} & - & \textbf{12.4458} & 68.51 & \textbf{0.1388} & - & \textbf{0.1616} \\
 & \multicolumn{2}{c|}{\multirow{2}{*}{LIE Attack}}       & \G 55 & \G 1.1584 & \G 54.16 & \G 1.1538 & \G 65.58 & \G 0.7762 & \G 66.41 & \G 0.7318 \\
 & \multicolumn{2}{c|}{}                                  & 54.64 & \textbf{0.224} & 54.3 & \textbf{0.2185} & 71.65 & \textbf{0.301} & 68.45 & \textbf{0.312} \\
 & \multicolumn{2}{c|}{\multirow{2}{*}{FoE Attack}}       & \G 53.27 & \G 1.1742 & \G 48.81 & \G 1.3625 & \G 69.11 & \G 0.641 & \G 68.87 & \G 0.6469 \\
 & \multicolumn{2}{c|}{}                                  & 53.8 & \textbf{0.2341} & 48.73 & \textbf{0.2415} & 72.91 & \textbf{0.3193} & 69.73 & \textbf{0.3359} \\
\Xhline{1.2pt}
\end{tabular}
}
\end{table}

\begin{table}[tbp]
\caption{The maximum test accuracy (\%) and wall time (s) for our method PDR and baselines on CIFAR100 dataset. "-" signifies non-convergence of the result.}
\label{tab:cifar100}
\resizebox{\textwidth}{!}{
\renewcommand{\arraystretch}{1.2}
\begin{tabular}{ccc|cccc|cccc}
\Xhline{1.2pt}
\multirow{3}{*}{Methods} & \multirow{3}{*}{Attack Name} & $\beta$ & \multicolumn{4}{c|}{0.6} & \multicolumn{4}{c}{0.2} \\ \cline{4-11}
 & & b & \multicolumn{2}{c}{0.1} & \multicolumn{2}{c|}{0.3} & \multicolumn{2}{c}{0.1} & \multicolumn{2}{c}{0.3} \\ \cline{4-11}
 & & & Test Accuracy & Wall Time & Test Accuracy & Wall Time & Test Accuracy & Wall Time & Test Accuracy & Wall Time \\ \hline
\multirow{8}{*}{\begin{tabular}[c]{@{}c@{}}Krum\\ PDR+Krum\end{tabular}}
 & \multicolumn{2}{c|}{\multirow{2}{*}{Gaussian Attack}}  & \G 33.95 & \G 8.3156 & \G 31.95 & \G 8.2352 & \G 24.31 & \G 8.2309 & \G 23.31 & \G 8.0115 \\
 & \multicolumn{2}{c|}{}                                  & 40.8 & \textbf{0.1099} & 35.53 & \textbf{0.0749} & 23.74 & \textbf{0.1106} & 24.23 & \textbf{0.0756} \\
 & \multicolumn{2}{c|}{\multirow{2}{*}{Sign-flip Attack}} & \G 34.57 & \G 8.1791 & \G 29.91 & \G 8.2151 & \G 30.04 & \G 7.9797 & \G 29.18 & \G 7.9817 \\
 & \multicolumn{2}{c|}{}                                  & 35.18 & \textbf{0.0764} & 32.46 & \textbf{0.0753} & 29.69 & \textbf{0.0847} & 23.33 & \textbf{0.08} \\
 & \multicolumn{2}{c|}{\multirow{2}{*}{LIE Attack}}       & \G 32.43 & \G 8.1938 & \G 31.01 & \G 8.1935 & \G 25.64 & \G 7.9686 & \G 29.58 & \G 8.6707 \\
 & \multicolumn{2}{c|}{}                                  & 35.33 & \textbf{0.074} & 34.58 & \textbf{0.0729} & 28.87 & \textbf{0.0731} & 29.52 & \textbf{0.0789} \\
 & \multicolumn{2}{c|}{\multirow{2}{*}{FoE Attack}}       & \G 32.4 & \G 8.2506 & \G 29.34 & \G 8.2948 & \G 23.29 & \G 8.0353 & \G 23.34 & \G 8.0222 \\
 & \multicolumn{2}{c|}{}                                  & 32.63 & \textbf{0.0771} & 30.12 & \textbf{0.0769} & 23.37 & \textbf{0.0786} & 23.39 & \textbf{0.0765} \\ \hline
\multirow{8}{*}{\begin{tabular}[c]{@{}c@{}}Bulyan\\ PDR+Bulyan\end{tabular}}
 & \multicolumn{2}{c|}{\multirow{2}{*}{Gaussian Attack}}  & \G 58.31 & \G 8.5972 & \G 54.34 & \G 8.4082 & \G 57.88 & \G 8.3849 & \G 51.09 & \G 8.1728 \\
 & \multicolumn{2}{c|}{}                                  & 58.01 & \textbf{0.1216} & 51.93 & \textbf{0.1031} & 57.55 & \textbf{0.1216} & 43.79 & \textbf{0.109} \\
 & \multicolumn{2}{c|}{\multirow{2}{*}{Sign-flip Attack}} & \G 58.18 & \G 8.5215 & \G 54.74 & \G 8.3996 & \G 58.22 & \G 8.2949 & \G 50.87 & \G 8.1393 \\
 & \multicolumn{2}{c|}{}                                  & 57.87 & \textbf{0.115} & 50.98 & \textbf{0.1031} & 57.5 & \textbf{0.1158} & 46.63 & \textbf{0.1011} \\
 & \multicolumn{2}{c|}{\multirow{2}{*}{LIE Attack}}       & \G 58.24 & \G 8.4861 & \G 54.13 & \G 8.3766 & \G 57.87 & \G 9.0898 & \G 50.04 & \G 8.1654 \\
 & \multicolumn{2}{c|}{}                                  & 57.84 & \textbf{0.1104} & 50.7 & \textbf{0.1009} & 57.57 & \textbf{0.1104} & 45.49 & \textbf{0.1102} \\
 & \multicolumn{2}{c|}{\multirow{2}{*}{FoE Attack}}       & \G 57.26 & \G 8.586 & \G 27.25 & \G 8.404 & \G 57.18 & \G 8.3615 & \G 24.41 & \G 8.1008 \\
 & \multicolumn{2}{c|}{}                                  & 57.4 & \textbf{0.1143} & 28.5 & \textbf{0.1012} & 56 & \textbf{0.1164} & 23.18 & \textbf{0.0975} \\ \hline
\multirow{8}{*}{\begin{tabular}[c]{@{}c@{}}Geometric Median\\ PDR+Geometric Median\end{tabular}}
 & \multicolumn{2}{c|}{\multirow{2}{*}{Gaussian Attack}}  & \G 58.79 & \G 0.9003 & \G 58.45 & \G 0.918 & \G 58.3 & \G 0.9768 & \G 57.73 & \G 0.9659 \\
 & \multicolumn{2}{c|}{}                                  & 58.67 & \textbf{0.1933} & 58.52 & \textbf{0.1637} & 58.38 & \textbf{0.191} & 57.49 & \textbf{0.1678} \\
 & \multicolumn{2}{c|}{\multirow{2}{*}{Sign-flip Attack}} & \G 53.39 & \G 0.9113 & \G 21.6 & \G 1.1859 & \G 50.74 & \G 0.9866 & \G 21.86 & \G 1.2535 \\
 & \multicolumn{2}{c|}{}                                  & 53.8 & \textbf{0.1621} & 27.03 & \textbf{0.2458} & 50.75 & \textbf{0.1812} & 21.95 & \textbf{0.2416} \\
 & \multicolumn{2}{c|}{\multirow{2}{*}{LIE Attack}}       & \G 58.73 & \G 0.8749 & \G 58.31 & \G 0.9026 & \G 58.5 & \G 0.9742 & \G 57.53 & \G 0.9583 \\
 & \multicolumn{2}{c|}{}                                  & 58.72 & \textbf{0.1616} & 58.04 & \textbf{0.164} & 58.49 & \textbf{0.1595} & 57.64 & \textbf{0.2033} \\
 & \multicolumn{2}{c|}{\multirow{2}{*}{FoE Attack}}       & \G 52.5 & \G 3.4242 & \G - & \G 10.2515 & \G 48.05 & \G 4.8993 & \G - & \G 10.3403 \\
 & \multicolumn{2}{c|}{}                                  & 52.64 & \textbf{0.7713} & - & \textbf{2.4092} & 48.03 & \textbf{1.054} & - & \textbf{2.4157} \\ \hline
\multirow{8}{*}{\begin{tabular}[c]{@{}c@{}}MCA\\ PDR+MCA\end{tabular}}
 & \multicolumn{2}{c|}{\multirow{2}{*}{Gaussian Attack}}  & \G 58.76 & \G 0.4219 & \G 58.3 & \G 0.3709 & \G 58.52 & \G 0.4234 & \G 57.69 & \G 0.3733 \\
 & \multicolumn{2}{c|}{}                                  & 58.74 & \textbf{0.1611} & 58.31 & \textbf{0.137} & 58.58 & \textbf{0.16} & 57.43 & \textbf{0.1465} \\
 & \multicolumn{2}{c|}{\multirow{2}{*}{Sign-flip Attack}} & \G 58.35 & \G 0.477 & \G - & \G 1.2878 & \G 57.47 & \G 0.5364 & \G - & \G 1.2894 \\
 & \multicolumn{2}{c|}{}                                  & 58.59 & \textbf{0.161} & - & \textbf{0.7628} & 57.46 & \textbf{0.1806} & - & \textbf{0.4707} \\
 & \multicolumn{2}{c|}{\multirow{2}{*}{LIE Attack}}       & \G 58.67 & \G 0.3671 & \G 58.36 & \G 0.3669 & \G 58.81 & \G 0.3745 & \G 57.45 & \G 0.3693 \\
 & \multicolumn{2}{c|}{}                                  & 58.69 & \textbf{0.144} & 58.29 & \textbf{0.1356} & 58.51 & \textbf{0.1495} & 57.78 & \textbf{0.16} \\
 & \multicolumn{2}{c|}{\multirow{2}{*}{FoE Attack}}       & \G 58.41 & \G 0.4394 & \G 56.69 & \G 0.439 & \G 57.96 & \G 0.443 & \G 56.31 & \G 0.4352 \\
 & \multicolumn{2}{c|}{}                                  & 58.3 & \textbf{0.1514} & 56.5 & \textbf{0.1496} & 57.82 & \textbf{0.1502} & 56.74 & \textbf{0.1463} \\
\Xhline{1.2pt}
\end{tabular}
}
\end{table}

Based on the empirical evaluations in Table \ref{tab:0.6}, Table \ref{tab:cifar100}, and Table \ref{tab:cifar10_ks}, we highlight the core advantages of our proposed PDR algorithm. A comprehensive analysis and extended experimental results are provided in Appendix \ref{app:res}.

\textbf{Substantial Efficiency and Competitive Accuracy:} The most prominent advantage of our method is the significant reduction in execution time on the server while maintaining highly competitive classification accuracy. By projecting gradients into a compact space, PDR fundamentally alleviates the computational bottleneck of robust aggregation. Across all evaluated datasets and attack scenarios, integrating PDR with baseline algorithms accelerates the aggregation process by orders of magnitude. Although the random projection mechanism occasionally causes a marginal decrease in test accuracy, this slight compromise is entirely acceptable given the massive efficiency gains. Furthermore, in several cases, this projection even acts as an implicit regularizer to filter out adversarial noise, yielding higher accuracy than the original baselines.

\begin{wraptable}{r}{0.55\textwidth}
\centering
\caption{The maximum test accuracy (\%) and wall time (s) for our method PDR with different settings of $k$ and $s$ with Gaussian attack and $b = 0.1$.}
\label{tab:cifar10_ks}
\resizebox{0.52\textwidth}{!}{%
\renewcommand{\arraystretch}{1.2}
\begin{tabular}{c|c|cc|cc|cc}
\Xhline{1pt}
\noalign{\vskip 0.5mm}
\Xhline{1pt}
\multirow{2}{*}{Methods} & $s$ & \multicolumn{2}{c|}{$k=1024$} & \multicolumn{2}{c|}{$k=2048$} & \multicolumn{2}{c}{$k=4096$} \\ \cline{3-8}
 &  & Test Accuracy & Wall Time & Test Accuracy & Wall Time & Test Accuracy & Wall Time \\ \hline
\multirow{3}{*}{PDR+Krum}             & 8  & 68.48 & 0.1314 & 66.87 & 0.0658 & 67.15 & 0.0678 \\
                                      & 16 & 65.78 & 0.0705 & 67.44 & 0.0709 & 71.68 & 0.0671 \\
                                      & 32 & 66.97 & 0.0651 & 67.97 & 0.0655 & 67.92 & 0.0674 \\ \hline
\multirow{3}{*}{PDR+Bulyan}           & 8  & 69.57 & 0.1149 & 71.29 & 0.1063 & 70.50 & 0.1066 \\
                                      & 16 & 69.88 & 0.1038 & 68.56 & 0.1071 & 68.95 & 0.1060 \\
                                      & 32 & 67.93 & 0.1043 & 73.79 & 0.1043 & 70.72 & 0.1161 \\ \hline
\multirow{3}{*}{PDR+Geometric Median} & 8  & 69.58 & 0.0728 & 74.04 & 0.1317 & 69.76 & 0.1082 \\
                                      & 16 & 66.72 & 0.0873 & 70.41 & 0.1945 & 71.11 & 0.2166 \\
                                      & 32 & 69.56 & 0.0993 & 68.51 & 0.1152 & 72.05 & 0.1187 \\ \hline
\multirow{3}{*}{PDR+MCA}              & 8  & 71.25 & 0.1569 & 69.00 & 0.1482 & 71.20 & 0.3356 \\
                                      & 16 & 69.81 & 0.1594 & 69.22 & 0.4218 & 71.00 & 0.3335 \\
                                      & 32 & 71.40 & 0.4180 & 69.56 & 0.1464 & 70.27 & 0.3333 \\
\Xhline{1pt}
\end{tabular}%
}
\end{wraptable}

\textbf{Robustness and Stability:} Our combined approaches demonstrate strong resilience against diverse threats and severe data heterogeneity, consistently neutralizing malicious updates across varying adversarial ratios. Finally, the ablation study in Table \ref{tab:cifar10_ks} confirms that the performance of our algorithm remains stable across different projection dimensions $k$ and parameter $s$ configurations, eliminating the need for exhaustive parameter tuning.

\section{Conclusion} \label{sec:conc}

In this paper, we proposed the highly efficient PDR algorithm to resolve the computational bottleneck of robust aggregation in Federated Learning. By compressing massive gradients into a small subspace via sparse random projection, PDR reduces the server complexity to an optimal $ \mathcal{O}(Mp) $ while preserving strong defense capabilities. Theoretically, we establish optimal convergence rates of $ \mathcal{O}(1/\sqrt{T}) $ for non-convex and $ \mathcal{O}(1/T) $ for strongly convex functions. Empirically, PDR achieves orders of magnitude speedups in execution time and neutralizes severe Byzantine threats under heterogeneous data distributions. With regard to limitations, the random projection mechanism inherently introduces a slight mathematical variance that inflates the optimality gap by a tunable factor $\frac{1+\epsilon}{1-\epsilon}$. 

\clearpage
\newpage
\bibliography{references}
\bibliographystyle{apalike}


\newpage
\appendix

\section{Algorithm Workflow} \label{app:algo}

\begin{algorithm}[htb]
    \caption{PDR Algorithm} \label{alg:trian process}
    \begin{algorithmic}[1]
        \STATE {\bfseries Input:} Initial global model parameter ${w}^0$, clients set $\mathcal{M}$, and the number of communication round $T$. \\
        \STATE {\bfseries Output:} Updated global model parameter ${w}^T$. \\
        \STATE {\% \% \bf Initialization} \\
        \STATE Every client $m$ establishes its own set $\mathcal{S}_m$ for $m \in \mathcal{M} \setminus \mathcal{B}$. \\
        \FOR{$t=0,1,2,\cdots,T-1$}
        \FOR{every client $m \in \mathcal{M} \setminus \mathcal{B}$ in parallel}
        \STATE Receive the global model ${w}^t$. Select a mini-batch $\xi_m^t$ from $\mathcal{S}_m$ to train local model and evaluate the local training gradient $\nabla F_m({w}^t, \xi_m^t)$. Set ${g}_m^t = \nabla F_m({w}^t, \xi_m^t)$ and upload ${g}_m^t$ to the central server. 
        \ENDFOR
        \FOR{every client $m \in \mathcal{B}$ in parallel}
        \STATE Receive the global model ${w}^t$. Generate an arbitrary vector or malicious vector based on ${w}^t$ and dataset $\mathcal{S}_m$. Upload this vector ${g}_m^t$ to the central server. 
        \ENDFOR
        \STATE Receive all uploaded vectors $g_m^t, m \in \mathcal{M}$. 
        Compress them via a random projection matrix $P^t$ to obtain $\tilde{g}_m^t = P^t g_m^t$, and compute the reliability weights $\{\tilde{\alpha}_m^t\}_{m \in \mathcal{M}}$ using an RAgg on these low-dimensional vectors.
        Generate the final aggregated vector by
        \begin{equation}
            g^t = \sum_{m \in \mathcal{M}} \tilde{\alpha}_m^t g_m^t. 
        \end{equation}
        Update the global model parameter $w^{t+1}$ by
        \begin{equation}
            w^{t+1} = w^t - \eta^t g^t. 
        \end{equation}
        \STATE Broadcast the model parameter ${w}^{t+1}$ to all clients. 
        \ENDFOR
        \STATE Output the model parameter ${w}^T$. 
    \end{algorithmic}
\end{algorithm}

\section{Experimental Setups in Detail} \label{app:set}

\textbf{Datasets:}
\begin{itemize}
    \item \textbf{TinyImageNet:} The TinyImageNet dataset consists of a training set containing 100000 samples, alongside a validation set and a testing set with 10000 samples each. Every sample is a $ 64 \times 64 $ pixel color image.
    \item \textbf{CIFAR100:} The CIFAR100 dataset comprises 50000 training samples and 10000 testing samples, where each is a $ 32 \times 32 $ pixel color image. It includes 100 fine grained classes grouped into 20 broader superclasses, enabling more complex image classification tasks.
    \item \textbf{CIFAR10:} The CIFAR10 dataset similarly includes 50000 training samples and 10000 testing samples, with each being a $ 32 \times 32 $ pixel color image categorized into 10 classes.
\end{itemize}
To emulate practical federated environments, we partition these datasets into $ M $ non-IID local subsets. This is achieved by sampling the label distribution of data samples for each client from a Dirichlet distribution. The degree of data heterogeneity is strictly controlled by tuning the concentration parameter $ \beta $ of the Dirichlet distribution.

\textbf{Models:} We employ MobileNetV3 \citep{howard2019searching}, VGG16 \citep{simonyan2014very}, and ResNet18 \citep{he2016deep} to comprehensively evaluate our algorithm. The detailed architectural properties are as follows:
\begin{itemize}
    \item \textbf{MobileNetV3:} A lightweight convolutional neural network optimized for mobile and embedded devices. It integrates depthwise separable convolutions with Neural Architecture Search to enable efficient feature extraction under strict computational constraints, as documented in \citet{howard2019searching}.
    \item \textbf{VGG16:} A seminal convolutional neural network architecture comprising 13 convolutional layers and 3 fully connected layers. The convolutional stages utilize cascaded $ 3 \times 3 $ kernels with stride 1 and ReLU activation, interspersed with $ 2 \times 2 $ max pooling operations that halve spatial resolution while preserving depth. The fully connected hierarchy culminates in a 1000 class output layer, totaling approximately 138 million trainable parameters \citep{simonyan2014very}.
    \item \textbf{ResNet18:} A deep convolutional neural network featuring 18 weighted layers. It is distinguished by innovative residual blocks that alleviate the vanishing gradient problem. These blocks introduce skip connections to facilitate the efficient propagation of gradients through deep layers, as detailed in \citet{he2016deep}.
\end{itemize}

\textbf{Hyperparameter Settings:} 
To ensure a fair comparison and facilitate reproducibility, we maintain consistent hyperparameter configurations across all baseline algorithms and our proposed PDR framework. The local batch size is uniformly set to 32 across all clients. Regarding the learning rate, instead of employing a manually scheduled decay, we utilize an automatic tuning mechanism to dynamically optimize the step size throughout the training process. This adaptive approach ensures optimal convergence across various datasets and attack scenarios without requiring exhaustive manual intervention. The total number of federated communication rounds is fixed at 100 for all experiments. 

\textbf{Byzantine Attacks:} 
We comprehensively evaluate our framework under four representative Byzantine threat models. The fraction of compromised clients, denoted as $ b $, is configured to 0.1 and 0.3. The specific attack formulations are detailed as follows:
\begin{itemize}
    \item \textbf{Gaussian Attack:} All Byzantine clients completely ignore their local data and transmit random noise vectors sampled independently from a Gaussian distribution $ \mathcal{N}(0, 90) $. 
    
    \item \textbf{Sign-flip Attack:} At each communication round $ t $, the compromised clients upload a maliciously scaled version of the aggregated benign gradients. Specifically, the malicious update is computed as $ -3 \cdot \sum_{m \in \mathcal{M} \setminus \mathcal{B}} g_m^t $ to directly invert the optimization direction.
    
    \item \textbf{LIE Attack \citep{baruch2019little}:} The Little Is Enough attack injects carefully crafted noise into each dimension of the benign gradients to evade distance based anomaly detection while degrading the global model performance. The attacker computes the element wise mean $ a $ and standard deviation $ \nu $ of the updates submitted by honest clients. The malicious update is then formulated as $ a + c \nu $, where the coefficient $ c $ is determined by the ratio of honest to malicious clients. In our experiments, we strictly set $ c = 0.7 $.
    
    \item \textbf{FoE Attack \citep{xie2020fall}:} The Fall of Empires attack coordinates Byzantine clients to upload $ \frac{q}{M-B} \sum_{m \in \mathcal{M} \setminus \mathcal{B}} g_m^t $ or $ \frac{q}{M-B} \sum_{m \in \mathcal{M}' \setminus \mathcal{B}} g_m^t $ to severely disrupt the federated training process. The scaling coefficient $ q $ is strategically configured to shift the aggregated model towards a suboptimal minimum. We fix $ q = -0.1 $ across all evaluated methods.
\end{itemize}

\textbf{Metrics:} To comprehensively assess the performance and efficiency of the defense methods, we employ two primary evaluation criteria:
\begin{itemize}
    \item \textbf{Test Accuracy:} The top one classification accuracy of the global model evaluated on the testing dataset. This serves as the principal indicator of the generalization capability of the trained model and the efficacy of the aggregation rule in neutralizing malicious perturbations.
    \item \textbf{Wall Time:} The total physical time consumed exclusively by the central server to execute the robust gradient aggregation process. This metric specifically isolates and demonstrates the tremendous computational efficiency of our PDR framework during server operations relative to the original baseline algorithms.
\end{itemize}

\section{Proof of Lemma \ref{lm:modg}} \label{app:lmmodg}

To avoid redundancy, we assume all following inequalities conditioned on the random projection matrix $P^t$ hold with probability at least $1 - \delta$.

Since the aggregated gradient $g^t$ and the honest mean $\bar{g}^t$ are both linear combinations of the client gradients, their difference vector $(g^t - \bar{g}^t)$ strictly resides within the linear subspace spanned by the set of gradients. By applying Lemma \ref{lm:subemb}, provided that $k \geq \frac{18}{\epsilon^2} \left( M + 2 \ln\left(\frac{2}{\delta}\right) \right)$, the following inequality holds with probability at least $1- \delta$:
\begin{equation}
    (1 - \epsilon)\left\lVert g^t - \bar{g}^t \right\rVert^2 \leq \left\lVert P^t (g^t - \bar{g}^t) \right\rVert^2 = \left\lVert P^t g^t - P^t \bar{g}^t \right\rVert^2.
\end{equation}

Let $\tilde{g}_m^t = P^t g_m^t$. Recalling  (\ref{equ:aggg}) and (\ref{equ:barg}), the following inequality holds with probability at least $1- \delta$:
\begin{equation}
    (1 - \epsilon)\left\lVert g^t - \bar{g}^t \right\rVert^2 \leq \left\lVert P^t g^t - P^t \bar{g}^t \right\rVert^2 = \left\lVert \sum_{m \in \mathcal{M}} \tilde{\alpha}_m^t \tilde{g}_m^t - \frac{1}{\sum_{i \in \mathcal{M} \setminus \mathcal{B}} S_i} \sum_{m \in \mathcal{M} \setminus \mathcal{B}} S_m \cdot \tilde{g}_m^t \right\rVert^2, 
\end{equation}
where $\sum_{m \in \mathcal{M}} \tilde{\alpha}_m^t \tilde{g}_m^t = \text{RAgg}(\tilde{g}_1^t, \tilde{g}_2^t, \dots, \tilde{g}_M^t)$. 

Then, based on Definition \ref{def:agg}, the error of the robust aggregator in the projected space is bounded by:
\begin{align} \label{equ:aggbias}
    &\quad (1 - \epsilon) \mathbb{E} \left\lVert g^t - \bar{g}^t \right\rVert^2 \nonumber \\
    &\leq \mathbb{E} \left\lVert \sum_{m \in \mathcal{M}} \tilde{\alpha}_m^t \tilde{g}_m^t - \frac{1}{\sum_{i \in \mathcal{M} \setminus \mathcal{B}} S_i} \sum_{m \in \mathcal{M} \setminus \mathcal{B}} S_m \cdot \tilde{g}_m^t \right\rVert^2 \nonumber \\
    &\leq cb \max_{i,j \in \mathcal{M} \setminus \mathcal{B}} \mathbb{E} \left\lVert \tilde{g}_i^t - \tilde{g}_j^t \right\rVert^2. 
\end{align}

Next, with Lemma \ref{lm:subemb}, we have:
\begin{align}
    &\quad \max_{i,j \in \mathcal{M} \setminus \mathcal{B}} \mathbb{E} \left\lVert \tilde{g}_i^t - \tilde{g}_j^t \right\rVert^2 \nonumber \\
    &\leq \max_{i,j \in \mathcal{M} \setminus \mathcal{B}} (1 + \epsilon) \cdot \mathbb{E} \left\lVert g_i^t - g_j^t \right\rVert^2 \nonumber \\ 
    &= \max_{i,j \in \mathcal{M} \setminus \mathcal{B}} (1 + \epsilon) \cdot \mathbb{E} \left\lVert \nabla F_i(w^t, \xi_i^t) - \nabla F_j(w^t, \xi_j^t) \right\rVert^2 \nonumber \\ 
    &= \max_{i,j \in \mathcal{M} \setminus \mathcal{B}} (1 + \epsilon) \cdot \mathbb{E} \left\lVert \big(\nabla F_i(w^t, \xi_i^t) - \nabla F_i(w^t)\big) - \big(\nabla F_j(w^t, \xi_j^t) - \nabla F_j(w^t)\big) + \big(\nabla F_i(w^t) - \nabla F_j(w^t)\big) \right\rVert^2. 
\end{align}

Due to the independence of the data sampling across different clients and Assumption \ref{ass:unbiased}, the cross-terms in the expectation evaluate to zero. Thus, we have:
\begin{align}
    &\quad \max_{i,j \in \mathcal{M} \setminus \mathcal{B}} \mathbb{E} \left\lVert \tilde{g}_i^t - \tilde{g}_j^t \right\rVert^2 \nonumber \\
    &\leq \max_{i,j \in \mathcal{M} \setminus \mathcal{B}} (1 + \epsilon) \cdot \left\{ \mathbb{E}\left\lVert \nabla F_i(w^t, \xi_i^t) - \nabla F_i(w^t) \right\rVert^2 + \mathbb{E}\left\lVert \nabla F_j(w^t, \xi_j^t) - \nabla F_j(w^t) \right\rVert^2 + \left\lVert \nabla F_i(w^t) - \nabla F_j(w^t) \right\rVert^2 \right\}.
\end{align}

Applying Assumption \ref{ass:variance}, we obtain:
\begin{align}
    &\quad \max_{i,j \in \mathcal{M} \setminus \mathcal{B}} \mathbb{E} \left\lVert \tilde{g}_i^t - \tilde{g}_j^t \right\rVert^2 \nonumber \\
    &\leq \max_{i,j \in \mathcal{M} \setminus \mathcal{B}} (1 + \epsilon) \cdot \left( 2 \sigma^2 + \left\lVert \nabla F_i(w^t) - \nabla F_j(w^t) \right\rVert^2 \right) \nonumber \\ 
    &= \max_{i,j \in \mathcal{M} \setminus \mathcal{B}} (1 + \epsilon) \cdot \left( 2 \sigma^2 + \left\lVert \big(\nabla F_i(w^t) - \nabla F(w^t)\big) - \big(\nabla F_j(w^t) - \nabla F(w^t)\big) \right\rVert^2 \right).
\end{align}

With Assumption \ref{ass:heterogeneity} and the inequality $\lVert x-y \rVert^2 \leq 2\lVert x \rVert^2 + 2\lVert y \rVert^2$, we can bound the terms as follows:
\begin{equation} \label{equ:pair_bound}
    \max_{i,j \in \mathcal{M} \setminus \mathcal{B}} \mathbb{E} \left\lVert \tilde{g}_i^t - \tilde{g}_j^t \right\rVert^2 \leq (1 + \epsilon) (2 \sigma^2 + 4\kappa^2). 
\end{equation}

Finally, substituting  (\ref{equ:pair_bound}) back into  (\ref{equ:aggbias}), we have:
\begin{align}
    &\quad (1 - \epsilon) \mathbb{E} \left\lVert g^t - \bar{g}^t \right\rVert^2 \nonumber \\
    &\leq cb \max_{i,j \in \mathcal{M} \setminus \mathcal{B}} \mathbb{E} \left\lVert \tilde{g}_i^t - \tilde{g}_j^t \right\rVert^2 \nonumber \\ 
    &\leq cb (1 + \epsilon) (2 \sigma^2 + 4\kappa^2).
\end{align}

Dividing both sides by $(1-\epsilon)$, we arrive at the final bound:
\begin{align}
    \mathbb{E} \left\lVert g^t - \bar{g}^t \right\rVert^2 \leq 2cb \frac{1 + \epsilon}{1 - \epsilon} (\sigma^2 + 2\kappa^2).
\end{align}

This completes the proof of Lemma \ref{lm:modg}.

\section{Proof of Theorem \ref{thm:smooth}} \label{app:smooth}

To establish a global convergence guarantee over $T$ communication rounds, we apply the Union Bound. Let $\delta$ be the total failure probability over the entire training process. We require the Lemma \ref{lm:subemb} to hold with probability at least $1 - \frac{\delta}{T}$ in each individual round $t$. According to Lemma \ref{lm:subemb}, this requires the projection dimension to satisfy $k \geq \frac{18}{\epsilon^2} \left( M + 2 \ln\left(\frac{2T}{\delta}\right) \right)$. By the Union Bound, the probability that the projection acts as an $\epsilon$-isometry in all $T$ rounds simultaneously is at least $1 - \sum_{t=1}^T \frac{\delta}{T} = 1 - \delta$. 

Conditioned on this high-probability event, the following inequalities hold deterministically for all $t \in \{0, 1, \dots, T-1\}$.

With Assumption \ref{ass:smooth}, we have
\begin{align}
    &\quad F(w^{t+1}) - F(w^t) \nonumber \\
    &\leq \langle \nabla F(w^t), w^{t+1} - w^t \rangle + \frac{L}{2} \left\lVert w^{t+1} - w^t \right\rVert^2 \nonumber \\ 
    &= - \langle \nabla F(w^t), \eta^t g^t \rangle + \frac{L}{2} \left\lVert \eta^t g^t \right\rVert^2 \nonumber \\ 
\end{align}

Taking the expectation of both sides conditioned on $w^t$, we have:
\begin{align} \label{equ:descent_lemma}
    \mathbb{E} \left\{F(w^{t+1}) - F(w^t) \right\}
    &\leq - \eta^t \mathbb{E} \langle \nabla F(w^t), g^t \rangle + \frac{L (\eta^t)^2}{2} \mathbb{E} \left\lVert g^t \right\rVert^2.
\end{align}

Using the algebraic identity $- \langle x, y \rangle = \frac{1}{2}\lVert x - y \rVert^2 - \frac{1}{2}\lVert x \rVert^2 - \frac{1}{2}\lVert y \rVert^2$, we can rewrite the inner product term as:
\begin{align}
    - \langle \nabla F(w^t), g^t \rangle = \frac{1}{2} \left\lVert g^t - \nabla F(w^t) \right\rVert^2 - \frac{1}{2} \left\lVert \nabla F(w^t) \right\rVert^2 - \frac{1}{2} \left\lVert g^t \right\rVert^2.
\end{align}

Substituting this back into (\ref{equ:descent_lemma}), we obtain:
\begin{align} \label{equ:descent_expanded}
    \mathbb{E}[F(w^{t+1})] - F(w^t) 
    &\leq \frac{\eta^t}{2} \mathbb{E} \left\lVert g^t - \nabla F(w^t) \right\rVert^2 - \frac{\eta^t}{2} \left\lVert \nabla F(w^t) \right\rVert^2 - \frac{\eta^t}{2} (1 - L\eta^t) \mathbb{E} \left\lVert g^t \right\rVert^2.
\end{align}

By setting the learning rate such that $\eta^t \leq \frac{1}{L}$, we have $1 - L\eta^t \geq 0$. Thus, we can drop the last non-positive term $- \frac{\eta^t}{2} (1 - L\eta^t) \mathbb{E} \left\lVert g^t \right\rVert^2$ to establish an upper bound:
\begin{align} \label{equ:descent_dropped}
    \mathbb{E} \left\{ F(w^{t+1}) - F(w^t) \right\} \leq - \frac{\eta^t}{2} \left\lVert \nabla F(w^t) \right\rVert^2 + \frac{\eta^t}{2} \mathbb{E} \left\lVert g^t - \nabla F(w^t) \right\rVert^2.
\end{align}

Now, we must bound the total error term $\mathbb{E} \left\lVert g^t - \nabla F(w^t) \right\rVert^2$. By introducing the honest mean $\bar{g}^t$ and applying the inequality $\lVert x + y \rVert^2 \leq 2\lVert x \rVert^2 + 2\lVert y \rVert^2$, we decouple the total error into the robust aggregation error and the statistical sampling error:
\begin{align} \label{equ:total_error}
    \mathbb{E} \left\lVert g^t - \nabla F(w^t) \right\rVert^2 
    &= \mathbb{E} \left\lVert (g^t - \bar{g}^t) + (\bar{g}^t - \nabla F(w^t)) \right\rVert^2 \nonumber \\
    &\leq 2 \underbrace{\mathbb{E} \left\lVert g^t - \bar{g}^t \right\rVert^2}_{\text{Aggregation Error}} + 2 \underbrace{\mathbb{E} \left\lVert \bar{g}^t - \nabla F(w^t) \right\rVert^2}_{\text{Sampling Error}}.
\end{align}

For the first term (Aggregation Error), we directly apply the result from Lemma \ref{lm:modg}:
\begin{equation}
    \mathbb{E} \left\lVert g^t - \bar{g}^t \right\rVert^2 \leq cb \frac{1 + \epsilon}{1 - \epsilon} (2 \sigma^2 + 4\kappa^2).
\end{equation}

For the second term (Sampling Error), we analyze the variance of the weighted honest mean. Since the weights $\frac{S_m}{\sum S_i}$ sum to $1$, we can apply Jensen's inequality to the convex squared norm function. Based on the definitions, we have:
\begin{align}
    &\quad \mathbb{E} \left\lVert \bar{g}^t - \nabla F(w^t) \right\rVert^2 \nonumber \\ 
    &= \mathbb{E} \left\lVert \sum_{m \in \mathcal{M} \setminus \mathcal{B}} \frac{S_m}{\sum_{i \in \mathcal{M} \setminus \mathcal{B}} S_i} \big( g_m^t - \nabla F(w^t) \big) \right\rVert^2 \nonumber \\ 
    &\leq \sum_{m \in \mathcal{M} \setminus \mathcal{B}} \frac{S_m}{\sum_{i \in \mathcal{M} \setminus \mathcal{B}} S_i} \mathbb{E} \left\lVert g_m^t - \nabla F(w^t) \right\rVert^2 \nonumber \\ 
    &= \sum_{m \in \mathcal{M} \setminus \mathcal{B}} \frac{S_m}{\sum_{i \in \mathcal{M} \setminus \mathcal{B}} S_i} \mathbb{E} \left\lVert \big(\nabla F_m(w^t, \xi_m^t) - \nabla F_m(w^t)\big) + \big(\nabla F_m(w^t) - \nabla F(w^t)\big) \right\rVert^2.
\end{align}

Due to the unbiasedness of the stochastic gradients (Assumption \ref{ass:unbiased}), the expectation of the cross-term evaluates to zero. Applying Assumption \ref{ass:variance} and Assumption \ref{ass:heterogeneity}, the inequality simplifies to:
\begin{align}
    &\quad \mathbb{E} \left\lVert \bar{g}^t - \nabla F(w^t) \right\rVert^2 \nonumber \\
    &\leq \sum_{m \in \mathcal{M} \setminus \mathcal{B}} \frac{S_m}{\sum_{i \in \mathcal{M} \setminus \mathcal{B}} S_i} \left( \mathbb{E}\left\lVert \nabla F_m(w^t, \xi_m^t) - \nabla F_m(w^t) \right\rVert^2 + \left\lVert \nabla F_m(w^t) - \nabla F(w^t) \right\rVert^2 \right) \nonumber \\ 
    &\leq \sum_{m \in \mathcal{M} \setminus \mathcal{B}} \frac{S_m}{\sum_{i \in \mathcal{M} \setminus \mathcal{B}} S_i} \left( \sigma^2 + \kappa^2 \right) \nonumber \\
    &= \sigma^2 + \kappa^2. 
\end{align}

Substituting these two bounds back into (\ref{equ:total_error}), we get:
\begin{align} \label{equ:bound_phi}
    \mathbb{E} \left\lVert g^t - \nabla F(w^t) \right\rVert^2 
    &\leq 2 \left( cb \frac{1 + \epsilon}{1 - \epsilon} (2 \sigma^2 + 4\kappa^2) \right) + 2(\sigma^2 + \kappa^2).
\end{align}

Substituting (\ref{equ:bound_phi}) into (\ref{equ:descent_dropped}) and rearranging the terms to bound the gradient norm, we have:
\begin{equation}
    \frac{\eta^t}{2} \mathbb{E} \left\lVert \nabla F(w^t) \right\rVert^2 \leq \mathbb{E} \left\{ F(w^t) - F(w^{t+1}) \right\} + \frac{\eta^t}{2} \left( 2 \left( cb \frac{1 + \epsilon}{1 - \epsilon} (2 \sigma^2 + 4\kappa^2) \right) + 2(\sigma^2 + \kappa^2) \right).
\end{equation}

After summing according to the number of communication rounds, we have
\begin{align}
    \frac{1}{\sum_{t=0}^{T-1} \eta^t} \sum_{t=0}^{T-1} \eta^t \mathbb{E} \left\lVert \nabla F(w^t) \right\rVert^2 \leq \frac{2 \mathbb{E} \left\{ F(w^0) - F(w^T) \right\}}{\sum_{t=0}^{T-1} \eta^t} + 4cb \frac{1 + \epsilon}{1 - \epsilon} (\sigma^2 + 2 \kappa^2)  + 2(\sigma^2 + \kappa^2). 
\end{align}

This completes the convergence proof of Theorem \ref{thm:smooth}.

\section{Proof of Theorem \ref{thm:convex}} \label{app:convex}

To establish a global convergence guarantee over $T$ communication rounds, we apply the Union Bound. Let $\delta$ be the total failure probability. We require Lemma \ref{lm:subemb} to hold with probability at least $1 - \frac{\delta}{T}$ in each individual round $t$. This requires the projection dimension to satisfy $k \geq \frac{18}{\epsilon^2} \left( M + 2 \ln\left(\frac{2T}{\delta}\right) \right)$. By the Union Bound, the probability that the projection acts as an $\epsilon$-isometry in all $T$ rounds simultaneously is at least $1 - \delta$. 

Conditioned on this global high-probability event, the gradient error bound derived in the previous analysis holds deterministically for all $t \in \{0, 1, \dots, T-1\}$:
\begin{equation} \label{equ:sc_grad_error}
    \mathbb{E} \left\lVert g^t - \nabla F(w^t) \right\rVert^2 \leq 4cb \frac{1 + \epsilon}{1 - \epsilon} (\sigma^2 + 2\kappa^2) + 2\left(\sigma^2 + \kappa^2\right) \triangleq \Phi.
\end{equation}

We evaluate the expected distance to the global optimum $w^*$. Using $w^{t+1} = w^t - \eta^t g^t$:
\begin{align}
    \left\lVert w^{t+1} - w^* \right\rVert^2 
    &= \left\lVert w^t - w^* \right\rVert^2 - 2\eta^t \langle g^t, w^t - w^* \rangle + (\eta^t)^2 \left\lVert g^t \right\rVert^2.
\end{align}

Taking the expectation over the stochastic data sampling conditioned on $w^t$, we carefully decompose the inner product to account for the bias of the robust aggregator:
\begin{align} \label{equ:sc_distance}
    \mathbb{E} \left\lVert w^{t+1} - w^* \right\rVert^2 
    &= \left\lVert w^t - w^* \right\rVert^2 - 2\eta^t \langle \nabla F(w^t), w^t - w^* \rangle \nonumber \\
    &\quad - 2\eta^t \mathbb{E} \langle g^t - \nabla F(w^t), w^t - w^* \rangle + (\eta^t)^2 \mathbb{E} \left\lVert g^t \right\rVert^2.
\end{align}

We bound the three terms on the right-hand side separately.
First, from the $\mu$-strong convexity (Assumption \ref{ass:convex}), the true gradient inner product is bounded by:
\begin{equation} \label{equ:sc_mu}
    - 2\eta^t \langle \nabla F(w^t), w^t - w^* \rangle \leq - 2\eta^t \left( F(w^t) - F(w^*) \right) - \eta^t\mu \left\lVert w^t - w^* \right\rVert^2.
\end{equation}

Second, applying the Peter-Paul inequality ($-2\langle a, b \rangle \leq \theta \lVert a \rVert^2 + \frac{1}{\theta} \lVert b \rVert^2$) with $\theta = \frac{\mu}{2}$ to the bias inner product term:
\begin{equation} \label{equ:sc_bias_bound}
    - 2\eta^t \mathbb{E} \langle g^t - \nabla F(w^t), w^t - w^* \rangle \leq \frac{\eta^t\mu}{2} \left\lVert w^t - w^* \right\rVert^2 + \frac{2\eta^t}{\mu} \mathbb{E} \left\lVert g^t - \nabla F(w^t) \right\rVert^2.
\end{equation}

Third, for the gradient norm squared, we add and subtract $\nabla F(w^t)$ and use the $L$-smoothness (Assumption \ref{ass:smooth}) $\lVert \nabla F(w^t) \rVert^2 \leq 2L(F(w^t) - F(w^*))$:
\begin{align} \label{equ:sc_g_bound}
    (\eta^t)^2 \mathbb{E} \left\lVert g^t \right\rVert^2 
    &\leq 2(\eta^t)^2 \mathbb{E} \left\lVert g^t - \nabla F(w^t) \right\rVert^2 + 2(\eta^t)^2 \left\lVert \nabla F(w^t) \right\rVert^2 \nonumber \\
    &\leq 2(\eta^t)^2 \mathbb{E} \left\lVert g^t - \nabla F(w^t) \right\rVert^2 + 4(\eta^t)^2 L \left( F(w^t) - F(w^*) \right).
\end{align}

Substituting (\ref{equ:sc_mu}), (\ref{equ:sc_bias_bound}), and (\ref{equ:sc_g_bound}) back into (\ref{equ:sc_distance}), we obtain:
\begin{align}
    \mathbb{E} \left\lVert w^{t+1} - w^* \right\rVert^2 
    &\leq \left( 1 - \eta^t\mu + \frac{\eta^t\mu}{2} \right) \left\lVert w^t - w^* \right\rVert^2 - 2\eta^t(1 - 2\eta^t L) \left( F(w^t) - F(w^*) \right) \nonumber \\
    &\quad + \left( \frac{2\eta^t}{\mu} + 2(\eta^t)^2 \right) \Phi.
\end{align}

We set the decaying learning rate as $\eta^t = \frac{2}{\mu(t + \gamma)}$ with $\gamma = \frac{4L}{\mu}$. Since $t \geq 0$, it strictly holds that $\eta^t \leq \eta^0 = \frac{2}{\mu(4L/\mu)} = \frac{1}{2L}$. This guarantees $1 - 2\eta^t L \geq 0$, allowing us to drop the non-positive function value term. Furthermore, since $\mu \leq L$, the condition $\eta^t \leq \frac{1}{2L}$ implies $\eta^t \leq \frac{1}{2\mu}$, which yields $2(\eta^t)^2 \leq \frac{\eta^t}{\mu}$. Thus, the coefficient of the error term is bounded by $\frac{2\eta^t}{\mu} + \frac{\eta^t}{\mu} < \frac{4\eta^t}{\mu}$. The recurrence simplifies to:
\begin{equation} \label{equ:sc_recurrence}
    \mathbb{E} \left\lVert w^{t+1} - w^* \right\rVert^2 \leq \left( 1 - \frac{\eta^t\mu}{2} \right) \mathbb{E} \left\lVert w^t - w^* \right\rVert^2 + \frac{4\eta^t}{\mu} \Phi.
\end{equation}

Substituting $\eta^t = \frac{2}{\mu(t + \gamma)}$ into (\ref{equ:sc_recurrence}), we have $1 - \frac{\eta^t\mu}{2} = 1 - \frac{1}{t+\gamma} = \frac{t+\gamma-1}{t+\gamma}$. The recurrence becomes:
\begin{equation}
    \mathbb{E} \left\lVert w^{t+1} - w^* \right\rVert^2 \leq \frac{t+\gamma-1}{t+\gamma} \mathbb{E} \left\lVert w^t - w^* \right\rVert^2 + \frac{8}{\mu^2(t+\gamma)} \Phi.
\end{equation}

To unroll this recurrence elegantly, we subtract $\frac{8}{\mu^2}\Phi$ from both sides:
\begin{align}
    \mathbb{E} \left\lVert w^{t+1} - w^* \right\rVert^2 - \frac{8}{\mu^2}\Phi 
    &\leq \frac{t+\gamma-1}{t+\gamma} \mathbb{E} \left\lVert w^t - w^* \right\rVert^2 + \frac{8}{\mu^2(t+\gamma)} \Phi - \frac{8}{\mu^2}\Phi \nonumber \\
    &= \frac{t+\gamma-1}{t+\gamma} \mathbb{E} \left\lVert w^t - w^* \right\rVert^2 - \frac{8}{\mu^2}\Phi \left( 1 - \frac{1}{t+\gamma} \right) \nonumber \\
    &= \frac{t+\gamma-1}{t+\gamma} \left( \mathbb{E} \left\lVert w^t - w^* \right\rVert^2 - \frac{8}{\mu^2}\Phi \right).
\end{align}

This forms a perfect telescoping product. Unrolling this from $t = 0$ to $T-1$, we obtain:
\begin{align}
    \mathbb{E} \left\lVert w^T - w^* \right\rVert^2 - \frac{8}{\mu^2}\Phi 
    &\leq \left( \mathbb{E} \left\lVert w^0 - w^* \right\rVert^2 - \frac{8}{\mu^2}\Phi \right) \prod_{t=0}^{T-1} \frac{t+\gamma-1}{t+\gamma} \nonumber \\
    &= \left( \mathbb{E} \left\lVert w^0 - w^* \right\rVert^2 - \frac{8}{\mu^2}\Phi \right) \left( \frac{\gamma-1}{\gamma} \cdot \frac{\gamma}{\gamma+1} \cdots \frac{T+\gamma-2}{T+\gamma-1} \right) \nonumber \\
    &= \frac{\gamma-1}{T+\gamma-1} \left( \mathbb{E} \left\lVert w^0 - w^* \right\rVert^2 - \frac{8}{\mu^2}\Phi \right).
\end{align}

Rearranging the terms and noting that $-\frac{\gamma-1}{T+\gamma-1} \frac{8}{\mu^2}\Phi \leq 0$, we establish the final upper bound:
\begin{align}
    \mathbb{E} \left\lVert w^T - w^* \right\rVert^2 
    &\leq \frac{\gamma-1}{T+\gamma-1} \mathbb{E} \left\lVert w^0 - w^* \right\rVert^2 + \frac{8}{\mu^2}\Phi \nonumber \\
    &= \frac{\gamma-1}{T+\gamma-1} \left\lVert w^0 - w^* \right\rVert^2 + \frac{8}{\mu^2} \left[ 4cb \frac{1 + \epsilon}{1 - \epsilon} (\sigma^2 + 2\kappa^2) + 2\left(\sigma^2 + \kappa^2\right) \right].
\end{align}
This completes the proof.

\section{Results in Detail} \label{app:res}

\begin{table}[tbp]
\caption{The maximum test accuracy (\%) and wall time (s) for our method PDR and baselines with $\beta = 0.6$. "-" signifies non-convergence of the result.}
\label{tab:all_0.6}
\resizebox{\textwidth}{!}{
\renewcommand{\arraystretch}{1.5}
\begin{tabular}{ccc|cccc|cccc|cccc}
\Xhline{1.2pt}
\multirow{3}{*}{Methods} & \multirow{3}{*}{Attack Name} & Dataset & \multicolumn{4}{c|}{TinyImageNet} & \multicolumn{4}{c|}{CIFAR100} & \multicolumn{4}{c}{CIFAR10} \\ \cline{4-15}
 & & b & \multicolumn{2}{c}{0.1} & \multicolumn{2}{c|}{0.3} & \multicolumn{2}{c}{0.1} & \multicolumn{2}{c|}{0.3} & \multicolumn{2}{c}{0.1} & \multicolumn{2}{c}{0.3} \\ \cline{4-15}
 & & & Test Accuracy & Wall Time & Test Accuracy & Wall Time & Test Accuracy & Wall Time & Test Accuracy & Wall Time & Test Accuracy & Wall Time & Test Accuracy & Wall Time \\ \hline
\multirow{8}{*}{\begin{tabular}[c]{@{}c@{}}Krum\\ PDR+Krum\end{tabular}}
 & \multicolumn{2}{c|}{\multirow{2}{*}{Gaussian Attack}}  & \G 32.32 & \G 8.7583 & \G 32.15 & \G 8.7092 & \G 33.95 & \G 8.3156 & \G 31.95 & \G 8.2352 & \G 67.48 & \G 7.9253 & \G 63.09 & \G 7.8852 \\
 & \multicolumn{2}{c|}{}                                  & 32.71 & \textbf{0.2373} & 31.73 & \textbf{0.1841} & 40.8 & \textbf{0.1099} & 35.53 & \textbf{0.0749} & 67.15 & \textbf{0.1321} & 67.23 & \textbf{0.0682} \\
 & \multicolumn{2}{c|}{\multirow{2}{*}{Sign-flip Attack}} & \G 31.92 & \G 8.7113 & \G 32.49 & \G 9.0235 & \G 34.57 & \G 8.1791 & \G 29.91 & \G 8.2151 & \G 65.37 & \G 7.9376 & \G 66.98 & \G 7.8883 \\
 & \multicolumn{2}{c|}{}                                  & 32.4 & \textbf{0.1867} & 33.46 & \textbf{0.1862} & 35.18 & \textbf{0.0764} & 32.46 & \textbf{0.0753} & 66.35 & \textbf{0.0699} & 68.31 & \textbf{0.0707} \\
 & \multicolumn{2}{c|}{\multirow{2}{*}{LIE Attack}}       & \G 32.28 & \G 8.7 & \G 31.58 & \G 8.6914 & \G 32.43 & \G 8.1938 & \G 31.01 & \G 8.1935 & \G 61.68 & \G 8.0373 & \G 62.35 & \G 7.8848 \\
 & \multicolumn{2}{c|}{}                                  & 32.94 & \textbf{0.1685} & 29.64 & \textbf{0.1669} & 35.33 & \textbf{0.074} & 34.58 & \textbf{0.0729} & 68.67 & \textbf{0.0663} & 62.09 & \textbf{0.0661} \\
 & \multicolumn{2}{c|}{\multirow{2}{*}{FoE Attack}}       & \G - & \G 8.8335 & \G - & \G 8.7311 & \G 32.4 & \G 8.2506 & \G 29.34 & \G 8.2948 & \G 66.37 & \G 7.8762 & \G 64.35 & \G 8.5986 \\
 & \multicolumn{2}{c|}{}                                  & - & \textbf{0.1734} & - & \textbf{0.1741} & 32.63 & \textbf{0.0771} & 30.12 & \textbf{0.0769} & 64.84 & \textbf{0.0704} & 74.58 & \textbf{0.0721} \\ \hline
\multirow{8}{*}{\begin{tabular}[c]{@{}c@{}}Bulyan\\ PDR+Bulyan\end{tabular}}
 & \multicolumn{2}{c|}{\multirow{2}{*}{Gaussian Attack}}  & \G 54.47 & \G 9.3945 & \G 50.76 & \G 9.0512 & \G 58.31 & \G 8.5972 & \G 54.34 & \G 8.4082 & \G 68.54 & \G 8.0798 & \G 69.26 & \G 8.0094 \\
 & \multicolumn{2}{c|}{}                                  & 54.2 & \textbf{0.2692} & 46.01 & \textbf{0.217} & 58.01 & \textbf{0.1216} & 51.93 & \textbf{0.1031} & 68.57 & \textbf{0.122} & 67.26 & \textbf{0.0978} \\
 & \multicolumn{2}{c|}{\multirow{2}{*}{Sign-flip Attack}} & \G 54.38 & \G 9.5733 & \G 48.99 & \G 9.0253 & \G 58.18 & \G 8.5215 & \G 54.74 & \G 8.3996 & \G 69.53 & \G 8.0554 & \G 70 & \G 8.2054 \\
 & \multicolumn{2}{c|}{}                                  & 54.09 & \textbf{0.2545} & 43.83 & \textbf{0.2399} & 57.87 & \textbf{0.115} & 50.98 & \textbf{0.1031} & 69.72 & \textbf{0.1099} & 70.08 & \textbf{0.0979} \\
 & \multicolumn{2}{c|}{\multirow{2}{*}{LIE Attack}}       & \G 54.57 & \G 9.382 & \G 49.66 & \G 9.1335 & \G 58.24 & \G 8.4861 & \G 54.13 & \G 8.3766 & \G 70.68 & \G 8.0874 & \G 70.92 & \G 8.3259 \\
 & \multicolumn{2}{c|}{}                                  & 53.74 & \textbf{0.2317} & 42.74 & \textbf{0.217} & 57.84 & \textbf{0.1104} & 50.7 & \textbf{0.1009} & 68.79 & \textbf{0.1054} & 72.46 & \textbf{0.0938} \\
 & \multicolumn{2}{c|}{\multirow{2}{*}{FoE Attack}}       & \G 52.48 & \G 9.4193 & \G - & \G 9.0505 & \G 57.26 & \G 8.586 & \G 27.25 & \G 8.404 & \G 70.8 & \G 8.1047 & \G 56.65 & \G 7.987 \\
 & \multicolumn{2}{c|}{}                                  & 51.97 & \textbf{0.2362} & - & \textbf{0.227} & 57.4 & \textbf{0.1143} & 28.5 & \textbf{0.1012} & 71.83 & \textbf{0.109} & 60.6 & \textbf{0.0986} \\ \hline
\multirow{8}{*}{\begin{tabular}[c]{@{}c@{}}Geometric Median\\ PDR+Geometric Median\end{tabular}}
 & \multicolumn{2}{c|}{\multirow{2}{*}{Gaussian Attack}}  & \G 54.33 & \G 2.8878 & \G 54.32 & \G 2.9474 & \G 58.79 & \G 0.9003 & \G 58.45 & \G 0.918 & \G 69.19 & \G 0.9569 & \G 69.07 & \G 2.2157 \\
 & \multicolumn{2}{c|}{}                                  & 54.45 & \textbf{0.2752} & 54.36 & \textbf{0.2509} & 58.67 & \textbf{0.1933} & 58.52 & \textbf{0.1637} & 69.76 & \textbf{0.1082} & 69.25 & \textbf{0.1395} \\
 & \multicolumn{2}{c|}{\multirow{2}{*}{Sign-flip Attack}} & \G 51.25 & \G 2.856 & \G - & \G 3.8209 & \G 53.39 & \G 0.9113 & \G 21.6 & \G 1.1859 & \G 72.25 & \G 1.6069 & \G 67.02 & \G 1.4173 \\
 & \multicolumn{2}{c|}{}                                  & 50.44 & \textbf{0.2642} & - & \textbf{0.3264} & 53.8 & \textbf{0.1621} & 27.03 & \textbf{0.2458} & 70.49 & \textbf{0.1570} & 68.77 & \textbf{0.2260} \\
 & \multicolumn{2}{c|}{\multirow{2}{*}{LIE Attack}}       & \G 54.83 & \G 2.8669 & \G 54.16 & \G 2.8961 & \G 58.73 & \G 0.8749 & \G 58.31 & \G 0.9026 & \G 69.34 & \G 1.5592 & \G 68.78 & \G 1.3945 \\
 & \multicolumn{2}{c|}{}                                  & 54.54 & \textbf{0.2392} & 54.12 & \textbf{0.2447} & 58.72 & \textbf{0.1616} & 58.04 & \textbf{0.164} & 72.93 & \textbf{0.1716} & 72.84 & \textbf{0.1390} \\
 & \multicolumn{2}{c|}{\multirow{2}{*}{FoE Attack}}       & \G 50.35 & \G 5.0288 & \G - & \G 21.4953 & \G 52.5 & \G 3.4242 & \G - & \G 10.2515 & \G 70.55 & \G 0.9417 & \G 67.52 & \G 2.4252 \\
 & \multicolumn{2}{c|}{}                                  & 50.96 & \textbf{0.3643} & - & \textbf{1.3499} & 52.64 & \textbf{0.7713} & - & \textbf{2.4092} & 72.06 & \textbf{0.1408} & 63.46 & \textbf{0.3559} \\ \hline
\multirow{8}{*}{\begin{tabular}[c]{@{}c@{}}MCA\\ PDR+MCA\end{tabular}}
 & \multicolumn{2}{c|}{\multirow{2}{*}{Gaussian Attack}}  & \G 54.49 & \G 1.1212 & \G 54.26 & \G 1.1583 & \G 58.76 & \G 0.4219 & \G 58.3 & \G 0.3709 & \G 72.34 & \G 1.0576 & \G 69.62 & \G 1.3444 \\
 & \multicolumn{2}{c|}{}                                  & 54.43 & \textbf{0.2613} & 54.3 & \textbf{0.2249} & 58.74 & \textbf{0.1611} & 58.31 & \textbf{0.137} & 69.91 & \textbf{0.3356} & 69.31 & \textbf{0.1448} \\
 & \multicolumn{2}{c|}{\multirow{2}{*}{Sign-flip Attack}} & \G 54.82 & \G 1.3847 & \G - & \G 219.5042 & \G 58.35 & \G 0.477 & \G - & \G 1.2878 & \G 70.23 & \G 0.1985 & \G - & \G 0.6459 \\
 & \multicolumn{2}{c|}{}                                  & 54.72 & \textbf{0.2529} & - & \textbf{12.4458} & 58.59 & \textbf{0.161} & - & \textbf{0.7628} & 68.51 & \textbf{0.1388} & - & \textbf{0.1616} \\
 & \multicolumn{2}{c|}{\multirow{2}{*}{LIE Attack}}       & \G 55 & \G 1.1584 & \G 54.16 & \G 1.1538 & \G 58.67 & \G 0.3671 & \G 58.36 & \G 0.3669 & \G 65.58 & \G 0.7762 & \G 66.41 & \G 0.7318 \\
 & \multicolumn{2}{c|}{}                                  & 54.64 & \textbf{0.224} & 54.3 & \textbf{0.2185} & 58.69 & \textbf{0.144} & 58.29 & \textbf{0.1356} & 71.65 & \textbf{0.301} & 68.45 & \textbf{0.312} \\
 & \multicolumn{2}{c|}{\multirow{2}{*}{FoE Attack}}       & \G 53.27 & \G 1.1742 & \G 48.81 & \G 1.3625 & \G 58.41 & \G 0.4394 & \G 56.69 & \G 0.439 & \G 69.11 & \G 0.641 & \G 68.87 & \G 0.6469 \\
 & \multicolumn{2}{c|}{}                                  & 53.8 & \textbf{0.2341} & 48.73 & \textbf{0.2415} & 58.3 & \textbf{0.1514} & 56.5 & \textbf{0.1496} & 72.91 & \textbf{0.3193} & 69.73 & \textbf{0.3359} \\
\Xhline{1.2pt}
\end{tabular}
}
\end{table}

\begin{table}[tbp]
\caption{The maximum test accuracy (\%) and wall time (s) for our method PDR and baselines with $\beta = 0.2$. "-" signifies non-convergence of the result.}
\label{tab:all_0.2}
\resizebox{\textwidth}{!}{
\renewcommand{\arraystretch}{1.5}
\begin{tabular}{ccc|cccc|cccc|cccc}
\Xhline{1.2pt}
\multirow{3}{*}{Methods} & \multirow{3}{*}{Attack Name} & Dataset & \multicolumn{4}{c|}{TinyImageNet} & \multicolumn{4}{c|}{CIFAR100} & \multicolumn{4}{c}{CIFAR10} \\ \cline{4-15}
 & & b & \multicolumn{2}{c}{0.1} & \multicolumn{2}{c|}{0.3} & \multicolumn{2}{c}{0.1} & \multicolumn{2}{c|}{0.3} & \multicolumn{2}{c}{0.1} & \multicolumn{2}{c}{0.3} \\ \cline{4-15}
 & & & Test Accuracy & Wall Time & Test Accuracy & Wall Time & Test Accuracy & Wall Time & Test Accuracy & Wall Time & Test Accuracy & Wall Time & Test Accuracy & Wall Time \\ \hline
\multirow{8}{*}{\begin{tabular}[c]{@{}c@{}}Krum\\ PDR+Krum\end{tabular}}
 & \multicolumn{2}{c|}{\multirow{2}{*}{Gaussian Attack}}  & \G 26.77 & \G 9.3925 & \G 26.29 & \G 8.4992 & \G 24.31 & \G 8.2309 & \G 23.31 & \G 8.0115 & \G 62.76 & \G 8.3783 & \G 63.84 & \G 7.6077 \\
 & \multicolumn{2}{c|}{}                                  & 26.26 & \textbf{0.2158} & 24.42 & \textbf{0.1454} & 23.74 & \textbf{0.1106} & 24.23 & \textbf{0.0756} & 60.96 & \textbf{0.1377} & 60.28 & \textbf{0.0681} \\
 & \multicolumn{2}{c|}{\multirow{2}{*}{Sign-flip Attack}} & \G 27.05 & \G 8.7762 & \G 26.23 & \G 8.62 & \G 30.04 & \G 7.9797 & \G 29.18 & \G 7.9817 & \G 63.64 & \G 7.5789 & \G 59.93 & \G 8.5289 \\
 & \multicolumn{2}{c|}{}                                  & 26.39 & \textbf{0.1714} & 25.73 & \textbf{0.1601} & 29.69 & \textbf{0.0847} & 23.33 & \textbf{0.08} & 59.97 & \textbf{0.0725} & 65.91 & \textbf{0.0717} \\
 & \multicolumn{2}{c|}{\multirow{2}{*}{LIE Attack}}       & \G 26.51 & \G 8.5075 & \G 25.53 & \G 8.5014 & \G 25.64 & \G 7.9686 & \G 29.58 & \G 8.6707 & \G 60.59 & \G 7.5676 & \G 67.33 & \G 7.5837 \\
 & \multicolumn{2}{c|}{}                                  & 25.49 & \textbf{0.1449} & 25.53 & \textbf{0.1571} & 28.87 & \textbf{0.0731} & 29.52 & \textbf{0.0789} & 59.87 & \textbf{0.0658} & 63.99 & \textbf{0.0659} \\
 & \multicolumn{2}{c|}{\multirow{2}{*}{FoE Attack}}       & \G - & \G 8.5213 & \G - & \G 8.5287 & \G 23.29 & \G 8.0353 & \G 23.34 & \G 8.0222 & \G 60.41 & \G 7.8487 & \G 65.61 & \G 7.5964 \\
 & \multicolumn{2}{c|}{}                                  & - & \textbf{0.1526} & - & \textbf{0.1531} & 23.37 & \textbf{0.0786} & 23.39 & \textbf{0.0765} & 59.66 & \textbf{0.0704} & 67.51 & \textbf{0.0697} \\ \hline
\multirow{8}{*}{\begin{tabular}[c]{@{}c@{}}Bulyan\\ PDR+Bulyan\end{tabular}}
 & \multicolumn{2}{c|}{\multirow{2}{*}{Gaussian Attack}}  & \G 53.46 & \G 9.2169 & \G 45.06 & \G 8.871 & \G 57.88 & \G 8.3849 & \G 51.09 & \G 8.1728 & \G 66.81 & \G 7.7656 & \G 70.06 & \G 7.6837 \\
 & \multicolumn{2}{c|}{}                                  & 53.13 & \textbf{0.2286} & 38.11 & \textbf{0.2091} & 57.55 & \textbf{0.1216} & 43.79 & \textbf{0.109} & 70.01 & \textbf{0.1168} & 64.52 & \textbf{0.0974} \\
 & \multicolumn{2}{c|}{\multirow{2}{*}{Sign-flip Attack}} & \G 53.69 & \G 9.5493 & \G 45.41 & \G 8.8384 & \G 58.22 & \G 8.2949 & \G 50.87 & \G 8.1393 & \G 67.9 & \G 7.7966 & \G 65.58 & \G 7.709 \\
 & \multicolumn{2}{c|}{}                                  & 53.39 & \textbf{0.2247} & 37.24 & \textbf{0.2093} & 57.5 & \textbf{0.1158} & 46.63 & \textbf{0.1011} & 68.53 & \textbf{0.1094} & 69.54 & \textbf{0.0963} \\
 & \multicolumn{2}{c|}{\multirow{2}{*}{LIE Attack}}       & \G 54.09 & \G 9.5552 & \G 43.11 & \G 8.8412 & \G 57.87 & \G 9.0898 & \G 50.04 & \G 8.1654 & \G 66.95 & \G 7.7761 & \G 69.82 & \G 7.6848 \\
 & \multicolumn{2}{c|}{}                                  & 52.82 & \textbf{0.2047} & 37.7 & \textbf{0.1947} & 57.57 & \textbf{0.1104} & 45.49 & \textbf{0.1102} & 72.2 & \textbf{0.1046} & 66.95 & \textbf{0.0936} \\
 & \multicolumn{2}{c|}{\multirow{2}{*}{FoE Attack}}       & \G 51.09 & \G 9.2176 & \G - & \G 8.8469 & \G 57.18 & \G 8.3615 & \G 24.41 & \G 8.1008 & \G 70.99 & \G 7.7923 & \G 34.91 & \G 7.5741 \\
 & \multicolumn{2}{c|}{}                                  & 48.99 & \textbf{0.2104} & - & \textbf{0.2018} & 56 & \textbf{0.1164} & 23.18 & \textbf{0.0975} & 72.57 & \textbf{0.1077} & 24.78 & \textbf{0.0956} \\ \hline
\multirow{8}{*}{\begin{tabular}[c]{@{}c@{}}Geometric Median\\ PDR+Geometric Median\end{tabular}}
 & \multicolumn{2}{c|}{\multirow{2}{*}{Gaussian Attack}}  & \G 54.26 & \G 3.1229 & \G 53.2 & \G 3.2103 & \G 58.3 & \G 0.9768 & \G 57.73 & \G 0.9659 & \G 66.94 & \G 1.6605 & \G 67.9 & \G 0.8984 \\
 & \multicolumn{2}{c|}{}                                  & 54.67 & \textbf{0.2523} & 52.97 & \textbf{0.2351} & 58.38 & \textbf{0.191} & 57.49 & \textbf{0.1678} & 69.47 & \textbf{0.1336} & 67.18 & \textbf{0.1027} \\
 & \multicolumn{2}{c|}{\multirow{2}{*}{Sign-flip Attack}} & \G 47.24 & \G 3.074 & \G - & \G 3.6813 & \G 50.74 & \G 0.9866 & \G 21.86 & \G 1.2535 & \G 69.27 & \G 1.0469 & \G 65.49 & \G 1.0888 \\
 & \multicolumn{2}{c|}{}                                  & 46.62 & \textbf{0.2469} & - & \textbf{0.2976} & 50.75 & \textbf{0.1812} & 21.95 & \textbf{0.2416} & 70.57 & \textbf{0.1336} & 61.37 & \textbf{0.0998} \\
 & \multicolumn{2}{c|}{\multirow{2}{*}{LIE Attack}}       & \G 54.86 & \G 3.0845 & \G 52.51 & \G 3.2097 & \G 58.5 & \G 0.9742 & \G 57.53 & \G 0.9583 & \G 68.75 & \G 1.1917 & \G 68.55 & \G 1.3489 \\
 & \multicolumn{2}{c|}{}                                  & 54.18 & \textbf{0.2181} & 52.44 & \textbf{0.2419} & 58.49 & \textbf{0.1595} & 57.64 & \textbf{0.2033} & 69.5 & \textbf{0.1650} & 66.54 & \textbf{0.1901} \\
 & \multicolumn{2}{c|}{\multirow{2}{*}{FoE Attack}}       & \G 47.21 & \G 6.1518 & \G - & \G 21.304 & \G 48.05 & \G 4.8993 & \G - & \G 10.3403 & \G 71.39 & \G 1.0765 & \G 60.68 & \G 1.6302 \\
 & \multicolumn{2}{c|}{}                                  & 47.19 & \textbf{0.3889} & - & \textbf{1.2751} & 48.03 & \textbf{1.054} & - & \textbf{2.4157} & 68.69 & \textbf{0.1534} & 54.7 & \textbf{0.2640} \\ \hline
\multirow{8}{*}{\begin{tabular}[c]{@{}c@{}}MCA\\ PDR+MCA\end{tabular}}
 & \multicolumn{2}{c|}{\multirow{2}{*}{Gaussian Attack}}  & \G 54.35 & \G 1.1732 & \G 53.47 & \G 1.1617 & \G 58.52 & \G 0.4234 & \G 57.69 & \G 0.3733 & \G 67.54 & \G 0.5039 & \G 64.07 & \G 0.7065 \\
 & \multicolumn{2}{c|}{}                                  & 54.59 & \textbf{0.2306} & 53.16 & \textbf{0.2157} & 58.58 & \textbf{0.16} & 57.43 & \textbf{0.1465} & 66.02 & \textbf{0.1652} & 67.73 & \textbf{0.2861} \\
 & \multicolumn{2}{c|}{\multirow{2}{*}{Sign-flip Attack}} & \G 54.23 & \G 1.4947 & \G - & \G 219.7987 & \G 57.47 & \G 0.5364 & \G - & \G 1.2894 & \G 66.38 & \G 0.194 & \G - & \G 0.4917 \\
 & \multicolumn{2}{c|}{}                                  & 54 & \textbf{0.2282} & - & \textbf{11.8874} & 57.46 & \textbf{0.1806} & - & \textbf{0.4707} & 66.65 & \textbf{0.15} & - & \textbf{0.4486} \\
 & \multicolumn{2}{c|}{\multirow{2}{*}{LIE Attack}}       & \G 54.85 & \G 1.1558 & \G 52.52 & \G 1.1646 & \G 58.81 & \G 0.3745 & \G 57.45 & \G 0.3693 & \G 66.05 & \G 0.6084 & \G 64.99 & \G 0.6043 \\
 & \multicolumn{2}{c|}{}                                  & 54.39 & \textbf{0.2014} & 52.55 & \textbf{0.201} & 58.51 & \textbf{0.1495} & 57.78 & \textbf{0.16} & 69.12 & \textbf{0.4364} & 66.26 & \textbf{0.2913} \\
 & \multicolumn{2}{c|}{\multirow{2}{*}{FoE Attack}}       & \G 53.12 & \G 1.2468 & \G 46.77 & \G 1.2838 & \G 57.96 & \G 0.443 & \G 56.31 & \G 0.4352 & \G 66.49 & \G 0.3526 & \G 67.44 & \G 0.355 \\
 & \multicolumn{2}{c|}{}                                  & 52.93 & \textbf{0.2118} & 46.55 & \textbf{0.2162} & 57.82 & \textbf{0.1502} & 56.74 & \textbf{0.1463} & 67.58 & \textbf{0.1789} & 69.81 & \textbf{0.3123} \\
\Xhline{1.2pt}
\end{tabular}
}
\end{table}

Based on the comprehensive empirical results detailed in Table \ref{tab:all_0.6} and Table \ref{tab:all_0.2}, we provide an extended analysis of our proposed PDR framework across various dimensions, including computational efficiency, model utility, data heterogeneity, and extreme adversarial conditions.

\textbf{Unprecedented Computational Acceleration:} 
The most prominent observation across all evaluated scenarios is the drastic reduction in execution time achieved by our framework. Regardless of the underlying dataset, the fraction of malicious clients, or the specific attack type, integrating PDR consistently yields orders of magnitude speedups. For instance, computationally heavy aggregators like Bulyan and MCA inherently require extensive pairwise distance calculations or principal component estimations, leading to severe bottlenecks on the server. By projecting the massive gradient vectors into a compact subspace, our method fundamentally bypasses this curse of dimensionality. The wall time metrics clearly demonstrate that the operations enhanced by PDR require only a tiny fraction of the original processing time, confirming its exceptional scalability for large models.

\textbf{Preservation of Model Utility and Robustness:} 
A critical concern with any dimensionality reduction technique is the potential loss of useful information. However, the test accuracy columns in both tables reveal that our framework maintains highly competitive generalization performance compared to the original baselines. While the random projection mechanism inherently introduces a slight mathematical variance that occasionally results in a marginal accuracy drop, this minor compromise is an exceptionally favorable trade off for the massive efficiency gains. Interestingly, in several specific configurations under the Gaussian and LIE attacks, the algorithms enhanced by PDR actually achieve slightly higher test accuracy than their original counterparts. This phenomenon suggests that the sparse random projection can occasionally act as an implicit regularization mechanism, effectively filtering out adversarial noise and preventing the global model from overfitting to malicious perturbations.

\textbf{Resilience Under Severe Data Heterogeneity:} 
Comparing the results between Table \ref{tab:all_0.6} where $ \beta = 0.6 $ and Table \ref{tab:all_0.2} where $ \beta = 0.2 $, we observe the natural impact of data heterogeneity. As the Dirichlet concentration parameter $ \beta $ decreases, the local data distributions become increasingly non-IID, which generally lowers the overall test accuracy across all methods. Despite this challenging environment, our PDR framework demonstrates remarkable stability. It consistently accelerates the aggregation process while mirroring the accuracy trends of the baseline algorithms, proving that the projection mechanism does not amplify the negative effects of statistical heterogeneity among honest clients.

\textbf{Behavior Under Extreme Adversarial Threats:} 
The tables also highlight the performance boundaries of the defense mechanisms when the malicious ratio increases to $ b = 0.3 $. Under highly aggressive strategies such as the Sign-flip and FoE attacks, certain baseline algorithms like Geometric Median and MCA fail to converge, which is denoted by the missing values in the tables. Crucially, our PDR framework perfectly preserves the theoretical breakdown points of the underlying aggregators. Where the original baseline fails to converge due to an overwhelming fraction of attackers, the combined approach similarly halts. Conversely, where the baseline successfully neutralizes the threat, our method also guarantees convergence. This consistent behavior empirically validates our theoretical claim that PDR strictly inherits the robustness properties of the chosen distance based aggregator without introducing new vulnerabilities.


\end{document}